%% file: 00-main.tex
\documentclass[sigconf, balance=false]{acmart}

\input{01-packages}

\input{02-top}

\begin{document}
\input{03-abstract}


\maketitle
\sloppy 
\input{10-intro}

\input{20-background}
\input{31-algorithm}
\input{30-expm}

\input{40-experiments}
    \input{43-bigfig}

    \input{44-results}

\input{50-uq}
\input{60-mia}

\input{70-math}
\input{80-conclusion}

\bibliography{refs}
\bibliographystyle{abbrv}


\appendix
\input{90-appendix}


\end{document}

%% file: 01-packages.tex
    \PassOptionsToPackage{numbers, square, sort&compress}{natbib}
    
\usepackage{popets}
\usepackage[utf8]{inputenc} 
\usepackage[T1]{fontenc}    
\usepackage{xcolor}         
\usepackage{hyperref}       
\hypersetup{
    colorlinks,
    linkcolor={red!50!black},
    citecolor={blue!50!black},
    urlcolor={blue!80!black}
}

\usepackage{url}            
\usepackage{nicefrac}       
\usepackage{microtype}      
\usepackage{graphicx}
\usepackage{subfigure}
\usepackage{ulem}
\usepackage{bm}
\usepackage{algorithm2e, algorithmic}

\usepackage{amsfonts, amsmath, amsthm, breqn}
\usepackage{mathtools}

\theoremstyle{plain}
\newtheorem{theorem}{Theorem}[section]
\newtheorem{proposition}[theorem]{Proposition}
\newtheorem{lemma}[theorem]{Lemma}
\newtheorem{corollary}[theorem]{Corollary}
\theoremstyle{definition}
\newtheorem{definition}[theorem]{Definition}

\theoremstyle{remark}

\newcommand{\R}{\mathbb{R}}

\renewcommand{\P}{\mathcal{P}}
\DeclareMathOperator*{\argmax}{argmax}
\DeclareMathOperator*{\argmin}{argmin}

\newcommand{\smalltilde}{\raise.17ex\hbox{$\scriptstyle\mathtt{\sim}$}} 
\newcommand{\smallsim}{\raise.17ex\hbox{$\scriptstyle\mathtt{\sim}$}} 

\usepackage[capitalize,noabbrev]{cleveref}

\usepackage{wrapfig}
\usepackage{longtable,tabu}

\usepackage{multirow}
\usepackage{multicol}
\usepackage{booktabs}
\usepackage{threeparttable}

\usepackage{caption}

\usepackage{paralist} 

\usepackage{rotating}
\usepackage{lipsum}

\usepackage{changepage}

\usepackage{soul}

\copyrightyear{YYYY}

\acmYear{YYYY}
\acmVolume{YYYY}
\acmNumber{X}
\acmDOI{XXXXXXX.XXXXXXX}
\acmISBN{}
\acmConference{Preprint}

%% file: 02-top.tex
\title[]{Are Normalizing Flows the Key to Unlocking the Exponential Mechanism? }\thanks{Vandy Tombs, Robert Bridges co-first authors with equal contributions. Both are corresponding authors.}
\thanks{Code is available,  \url{https://github.com/bridgesra/expm_nf_mimic3_results_code}.}
\thanks{Special thanks to Cory Watson for making the big machines just work, to Adam Spannaus and Vishakh Gopu for their
wisdom with PyTorch, and to Ravi Tandon for insightful conversations.}
\thanks{ This manuscript has been co-authored by UT-Battelle, LLC under Contract No. DE-AC05-00OR22725 with the US DOE. The United States Government retains and the publisher, by accepting the article for publication, acknowledges that the United States Government retains a non-exclusive, paid-up, irrevocable, world-wide license to publish or reproduce the published form of this manuscript, or allow others to do so, for United States Government purposes. The DOE will provide public access to these results of federally sponsored research in accordance with the DOE Public Access Plan (\url{http://energy.gov/downloads/doe-public-access-plan}).}

\author{Robert A. Bridges}
\affiliation{%
  \institution{Oak Ridge National Laboratory}
  \city{Oak Ridge}
  \state{TN}
  \country{USA}}
\email{bridgesra@ornl.gov}
\email{robert.anthony.bridges@gmail.com}

\author{Vandy J. Tombs}
\affiliation{%
  \institution{Oak Ridge National Laboratory}
  \city{Oak Ridge}
  \state{TN}
  \country{USA}}
\email{tombsvj@ornl.gov}

\author{Christopher B. Stanley}
\affiliation{%
  \institution{Oak Ridge National Laboratory}
  \city{Oak Ridge}
  \state{TN}
  \country{USA}}
\email{stanleycb@ornl.gov}

%% file: 03-abstract.tex
\begin{abstract}
The Exponential Mechanism (ExpM), designed for private optimization, has been historically sidelined from use on continuous sample spaces, as it requires sampling from a generally intractable density, and, to a lesser extent, bounding the sensitivity of the objective function. 
Any differential privacy (DP) mechanism can be instantiated  as ExpM, and ExpM poses an elegant solution for private machine learning (ML) that bypasses inherent inefficiencies of DPSGD. 
This paper seeks to operationalize ExpM for private optimization and ML by using an auxiliary Normalizing Flow (NF)---an expressive deep network for density learning---to approximately sample from ExpM density. 
The new method, ExpM+NF, provides an alternative to SGD methods for model training. 
We prove a sensitivity bound for the $\ell^2$ loss function under mild hypotheses permitting ExpM use with any sampling method. 
To test model training feasibility, we present results on MIMIC-III health data comparing (non-private) SGD, DPSGD, and ExpM+NF training methods' accuracy and training time. 
We find that a model sampled from ExpM+NF is nearly as accurate as non-private SGD, more accurate than DPSGD, and ExpM+NF trains faster than Opacus' state-of-the-art professional DPSGD implementation.
Unable to provide a privacy proof for the NF approximation, we present empirical results to investigate privacy including the LiRA membership inference attack of Carlini et al. and the recent privacy auditing lower bound method of Steinke et al. 
Our findings suggest ExpM+NF provides more privacy than non-private SGD, but not as much as DPSGD, although many attacks are impotent against any model. 
Ancillary benefits of this work include pushing the state of the art of privacy and accuracy on MIMIC-III healthcare data, exhibiting the use of ExpM+NF for Bayesian inference, showing the limitations of empirical privacy auditing in practice, and providing several privacy theorems applicable to distribution learning. 
Future directions for proving privacy of ExpM+NF and  generally revisiting ExpM with modern sampling techniques are given. 
\end{abstract}

%% file: 10-intro.tex
\begin{figure*}
    \centering
    \includegraphics[width=\textwidth]{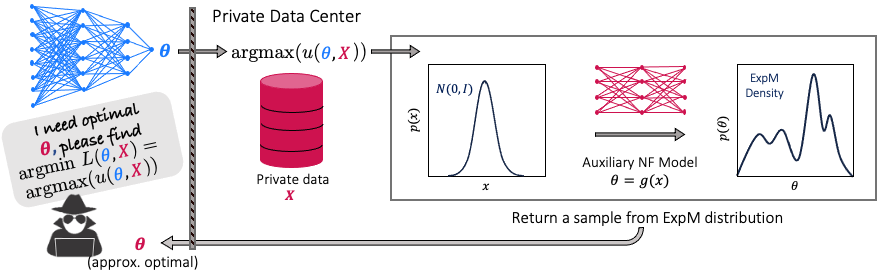}
    \caption{Figure depicts a to-be-publicized model trained on private data. New technique, ExpM+NF, trains an auxiliary Normalizing Flow (NF) model to approximately sample from the otherwise intractable Exponential Mechanism (ExpM) density to produce near-optimal model parameters $\theta$.} 
    \label{fig:expm+nf}
    \vspace{-.25cm}
\end{figure*}

\section{Introduction} 
\label{sec:intro}
Winning the 2017 G{\"o}del prize, Differential Privacy (DP) has become an accepted standard for privacy--preserving analytics including machine learning (ML) applications \cite{xiong2014survey, ponomareva2023how, dajani2017modernization, appleDP, erlingsson2014rappor, ding2017collecting}. 
In practice, differentially private ML is achieved almost exclusively by Differentially Private Stochastic Gradient Descent (DPSGD) \cite{abadi2016deep, ponomareva2023how}, and DPSGD is considered the research state of the art as well. 
To achieve a privacy guarantee, DPSGD clips and adds Gaussian noise to each gradient during SGD training, and upon the conclusion of training, the many privacy guarantees (one per step) are collected into a single, overarching guarantee. 
Aiding DPSGD's widespread adoption is 15 years of work on composition theorems \cite{dwork2007sequential, dwork2010advanced, kairouz2015composition, dong2021gaussian} which include methods that tighten the privacy bounds for mechanisms that utilize subsampling (e.g., ML training with batches) \cite{abadi2016deep, mironov2017renyi}, and clever techniques that compose privacy curves using the privacy random variable \cite{gopi2021numerical}. 

Yet, in the conventional scenario of training ML on a single private dataset, DPSGD is wasteful by design---privacy is expended so that all gradients may be publicized, while all that is required is the trained model. 
This results in cumulative privacy degradation with each training step, of which thousands are usually needed.
Less problematic, but still not ideal, is that DPSGD can only provide approximate DP, i.e., $\delta>0$, due to relying on composition theorems and on the Gaussian Mechanism, although it does provide ``graceful'' (continuous) $\varepsilon-\delta$ tradeoffs (see Definition \ref{defn:dp}).
Despite 15 years of fruitful research advancing the composition theorems, sub-sampling methods, and implementation techniques, adequate accuracy and privacy have often proved unattainable in practice \cite{googleblog2022, suriyakumar2021chasing}.

Conversely, the Exponential Mechanism (ExpM), developed in 2007 for private optimization \cite{mcsherry2007mechanism}, has been historically sidelined from ML training by two  obstructions. 
Primarily, sampling from the required distribution is computationally infeasible except for finite or simple domains, reducing ExpM to multinomial sampling \cite{friedman2010data, kapralov2013differentially, vinterbo2012differentially}.
Notwithstanding these examples, ExpM is historically intractable for most ML training and more generally any application involving a high--dimensional continuous sample space \cite{kasiviswanathan2011what, ji2014differential, minami2016differential,  near2022programming}.
Secondly, a sensitivity bound must be proven for the given utility (loss) function, which is difficult in general, e.g., \cite{minami2016differential}. 
Notably, ExpM furnishes strong DP ($\delta = 0$) and has a proven guarantee on its utility (accuracy) \cite{mcsherry2007mechanism}. 
Further, sampling a model from ExpM bypasses the incremental decline in privacy suffered by DPSGD; hence, ExpM's application to differentially private ML is potentially transformational, not just in method but in the privacy-accuracy tradeoff afforded. 
Additoinally, it is known that any DP mechanism can be instantiated as ExpM \cite{mcsherry2007mechanism}. 
In short, work to unlock ExpM promises large impact for differentially privacy, especially in ML application, and hence, research to discover how to operationalize ExpM is due.

Meanwhile, sampling methods have progressed independently. 
In particular, Normalizing Flows (NFs), expressive deep networks for approximating intractable distributions, have emerged as  alternatives to Markov Chain Monte Carlo (MCMC) techniques  that provide rapid sampling once trained, bypass inefficiencies of MCMC, and are generally effective on high--dimensional settings \cite{gabrie2021efficient, papamakarios2021normalizing}.
Hence the pairing of ExpM with an NF as a sampling methodology is a now-available solution that can potentially unlock the ExpM!

In this paper, we introduce ExpM+NF, in particular, as a general method of model training, depicted in Figure \ref{fig:expm+nf}. 
Instead of an SGD scheme, we leverage ExpM to optimize the loss function, and employ an auxiliary NF model for sampling (ExpM+NF). 
As our results show, it is an effective method providing very accurate results, but we are unable to provide a proof of privacy for the NF approximation of ExpM's density. 
By design ExpM+NF would eliminate the iterative privacy loss observed in DPSGD and offers potential for a strong, user-specified privacy guarantee ($\varepsilon >0, \delta = 0$) with near optimal parameters.

Below, we document our research to explore if and exactly how one may train an ML model via ExpM+NF; we examine its computational expense, accuracy, and privacy. 
We include positive and negative results, as well as identified limitations and potentially fruitful future directions. 
Overall, we believe the idea of using modern sampling methods with ExpM  has the potential to transform differentially private ML. 
Our hope is that the directions put forth in this document are fodder for this research community to identify a successful path in this regard.

\subsection{Contributions and Findings}

\textbf{Contribution 1:} We prove a sensitivity bound for $\ell^2$ (equiv. MSE) loss for supervised learning and regression tasks with bounded targets (Theorem \ref{thm:sensitivity}). 
This accommodates classification tasks via binary/one--hot encoding, permitting use of ExpM with any sampling technique.

\textbf{Contribution 2:} We exhibit that ExpM+NF is a viable ML training method (Section \ref{sec:experiments}). 
We present experiments comparing (non-private) SGD, DPSGD, and ExpM+NF methods to train logistic regression (LR, $\smallsim$ 2500 parameters) and GRU-D (a recurrent deep neural network, $\smallsim$15-100K parameters) models, on MIMIC-III (tabular electronic health) data \cite{johnson2016mimic}. 
(Here we treat $\tilde\varepsilon$ as a training parameter and do not claim any privacy guarantees for ExpM+NF).
Our results are exceptionally positive. 
ExpM+NF achieves an Area Under the ROC curve (AUC) greater than 94\% of the non-private baseline for $\tilde\varepsilon$ as small as $1\mathrm{e}{-3}$ in every experiment, and is far more accurate than DPSGD for small $\tilde\varepsilon$. 
We present computational benchmarking results showing our ExpM+NF implementation  is faster than Opacus' \cite{pytorchOpacus} implementation of DPSGD.
Consequently, investigation of the privacy of ExpM+NF is due---if privacy of ExpM+NF is attained then this method substantially improves the state of the art. 

\textbf{Contribution 2.b}: An ancillary benefit of the previous experiment is that our DPSGD implementation is the state of the art (in privacy and accuracy) on the MIMIC-III ICU binary prediction benchmarks. See Appendix \ref{sec:appendix-dpsgd-advancements}. 

\textbf{Contribution 2.c}: Using ExpM+NF is essentially training the NF to approximate a posterior distribution. In the non-private scenario, ExpM+NF facilitates Bayesian inference and uncertainty quantification (UQ) as it allows rapid sampling of parameters from the posterior. A simple example is exhibited, see Appendix \ref{sec:appendix-uq}.

\textbf{Contribution 3}: Just as accuracy on a held-out test set is used to gauge accuracy, we implement membership inference attacks (MIA) against (non-private) SGD, DPSGD, and ExpM+NF to gain empirical results on the methods' privacy. 
Using the (state of the art MIA) Likelihood Ratio Attack (LiRA)  \cite{carlini2022membership} our conclusions vary per model. 
When using the LR model, we find that  DPSGD is essentially unattackable, while ExpM+NF is less resilient than DPSGD but more so than non-private SGD. 
Our LiRA results for the GRU-D model are inconclusive as all models are essentially unattackable.
In short, the evidence is that ExpM+NF provides some privacy, but not as much as DPSGD for the same $\tilde \varepsilon$ parameter, yet it provides better accuracy for the same $\tilde \varepsilon$. 

\textbf{Contribution 4:} Recent exciting  work by Steinke et al. \cite{steinke2024privacy} proves a lower bound on the privacy guarantee $\varepsilon$, that holds with known likelihood, and is computable from the accuracy of a black-box attack seeking to identify known ``canaries'' in the training set. 
We apply this  privacy auditing methodology, but find that in all cases the lower bound is 0. 
Our results provide a glimpse at the difficulties with privacy auditing in practice---if a model is resistant to the attack used by Steinke et al., then the auditing methodology is of little help.

\textbf{Contribution 5:} Our final contributions are  mathematical advancements towards a formal differential privacy proof for ExpM+NF. 
While privacy of ExpM+NF is not proven, we present more widely applicable math results, and future methods of proof as well as obstructions identified. 
Specific advancements include the following: 
We define a metric $d_{DP}$ on probability densities so that convergence in this metric preserves $\varepsilon$-DP, and show it is strictly stronger than uniform convergence, which can provide ($\varepsilon, \delta>0$)-DP under appropriate hypothesis. (This shows the need for the new metric). 
We provide a Privacy Squeeze Theorem for proving privacy of a mechanism whose density is sandwiched between two densities of known privacy and similarly a privacy Projection Theorem. 

\subsection{Limitations}
\label{sec:limitations}
There are three primary limitations with ExpM+NF, currently. 
First and foremost is the lack of privacy proof for ExpM+NF. 
(Interestingly, works on MCMC sampling methods with ExpM \cite{lin2024tractable, gopi2022private} provide privacy proofs with no implementation or testing, so feasibility of the methods are unknown). 
The primary obstruction and need is a proof that for any training dataset  $X\in D$ (in our universe of datasets), the trained NF must lie either close to the target density or in some relationship to densities that have known privacy bounds. 
Theorems and formal definitions to make these notions rigorous are the topic of Section \ref{sec:math}, along with future directions for research. 
Second, ExpM+NF requires a sensitivity bound for each loss function, and thus far we have only provided such a proof for $\ell^2$ (equiv. mean squared error) loss; while this does not prove limiting in our experiments, future research is required for application to other loss functions. 
 By comparison, DPSGD is readily applicable to theoretically any loss function. 
Third, ExpM+NF requires a completely different training regime (training an auxiliary NF model whose loss function depends the desired model and its loss function); hence, custom software engineering is currently required for use on each model architecture (e.g., on LR, GRU-D, ...). 

We also identify experimental limitations. 
Hyperparameter searching was required for all methods and we ignored privacy of hyperparameter searching which is not reealistic. This choice was intentional---we aim to compare the very best DPSGD and ExpM+NF can perform. 
The limited abilities of both the empirical privacy auditing experiments (see Section \ref{sec:mia} containing results of the  LiRA membership inference attack of Carlini et al.  and black--box attack of Steinke et al.) narrow the aperture for empirically investigating ExpM+NF's privacy. 
Aside from impotence of these attacks, a degredation of defenses as the $\varepsilon$ increased is expected for both ExpM+NF and DPSGD and is not observed. 
This may be caused in part by the size and any obfuscation (purposefully before release by the curators, or inadvertently in the benchmark preparation pipeline, e.g., imputation of missing results) performed on the MIMIC-III data. 
Optimistically, the poor results of these attacks are good news for the health data community!

%% file: 20-background.tex
\section{Background}
\label{sec:background}

A more thorough discussion of all background appears in Appendix \ref{sec:appendix-background}.
We denote the real numbers as $\mathbb{R}$; the universe from which datasets $X, X'$ are drawn as $D$ (so $X, X' \in D$); and use $X, X'$ to denote neighboring datasets, meaning they are identical except for at most one data point (one row). 
The function $u:\Theta \times D \to \R$ is a utility function we wish to maximize privately over $\Theta$.
In the setting of machine learning, a dataset from $D$ is $(X,Y)$ (features and targets, respectively) in lieu of simply $X$, and $\Theta$ denotes the space of parameters for an ML model $\hat{y}$. 
Given loss function $L(X, Y,  \theta)$, the utility is defined as $u: = \phi\circ L$, with $\phi$ strictly decreasing, and usually $\phi(t) = -t$. 
A mechanism $f(\theta, X)$ provides a probability distribution over $\theta$ for each $X\in D$. 
We use $p$ for probability densities and capital $P$ to denote the induced probability measure, i.e., for measurable set $A, P(A):= \int_A p(\theta)d\theta$.

\subsection{Differential Privacy Background}
\begin{definition}[Differential Privacy] 
\label{defn:dp}
A mechanism $f:\Theta \times D\to\mathbb{R}$ 
satisfies $(\varepsilon, \delta)$-Differential Privacy or is $(\varepsilon, \delta)$-DP if and only if 
$P(f(\theta, X)\in A)  \leq P(f(\theta, X') \in A) e^\varepsilon + \delta$
for all neighboring $X, X'\in D, \theta\in \Theta, \varepsilon \geq0$ and for all measurable sets $A$. When $\delta = 0$, we say $f$ is $\varepsilon$-DP.
\end{definition} 

The $\varepsilon$ term is an upper bound on the privacy loss. Approximate DP, or $(\varepsilon, \delta)$-DP, provides a relaxed definition---$\delta$ is the probability that the mechanism $f$ is not guaranteed to be $\varepsilon$-DP. 
This ``probability of failure'' interpretation has led many to suggest that  $\delta$ should be no larger than $1/|X|$ \cite{dwork2014algorithmic}. 

To achieve $(\varepsilon, \delta)$-DP in practice, one calibrates the randomness (mechanism's variance) by how sensitive the output of the mechanism is across all neighboring datasets. 

\begin{definition}[Sensitivity]
\label{defn:sensitivity}
The sensitivity of a mechanism $f:\Theta\times D\to \R $ is  $s = \sup |f(\theta,X)-f(\theta, X')|$ where the supremum  is over neighboring $X,X' \in D$, and all $\theta \in \Theta$. (If multiple mechanisms are present, say $f_1, f_2$ we will use $\Delta(f_1), \Delta(f_2)$ for their respective sensitivities rather than $s$. 
\end{definition}


\begin{definition}[ExpM]
\label{defn:expm}
For utility function $u: \Theta \times D \to \R$ with sensitivity $s$, the Exponential Mechanism (ExpM) returns $\theta \in \Theta$ with likelihood  $p_{ExpM}(\theta; X) \propto \exp(\varepsilon u(\theta, X)/(2s))$.
\end{definition}
\begin{theorem}[ExpM satisfies $\varepsilon$-DP  \cite{mcsherry2007mechanism, dwork2014algorithmic}] 
The Exponential Mechanism (ExpM) i.e., sampling $\theta \sim p_{ExpM}(\theta, X) \propto \exp(\varepsilon u(\theta, X) / (2s))$ where utility function $u$ has sensitivity $s$, provides $\varepsilon$-DP.
\end{theorem}
The original treatment of ExpM \cite{mcsherry2007mechanism} furnishes utility theorems---rigorous guarantees that bound how far ExpM samples can be from the optima.

\subsection{Normalizing Flows (NFs)}
\label{sec:nf}
We refer readers to the survey of Papamakarios et al. \cite{papamakarios2021normalizing}. 
An NF model is a neural network, $g=g_k \circ  ...  \circ g_1: \R_z^n \to \R_\theta^n$, with each layer $g_i:\R^n \to \R^n$, called a flow, that is invertible and differentiable. 
We let $\beta_{\text{NF}}$ denote the parameters of $g$ to be learned, and suppress this notation for simplicity unless needed. 
Notably, $g$ (and by symmetry, $g^{-1}$) preserves probability densities: 
if $p_z$ is a probability density, then $\theta \mapsto p_z( g^{-1}(\theta) ) | \det J_{g^{-1}} (\theta) |$ is also.

Suppose the target density, $p^*_\theta$ can be computed up to a constant factor, but  we cannot sample from $p^*_\theta$.  
Let the base density $p_z$ be a tractable density over $\R^n$.  
Then $g$ can be trained to transform $p_z$ into $p_\theta$ so that $p_\theta \thickapprox p^*_{\theta}$ as  follows:  
iteratively sample a batch $z_i \sim p_z, i = 1, ..., N$ and minimize the empirical KL divergence of the base density $p_z$ with the pull-back of the target density, namely $ p^*_\theta(g(z))| \det J_g(z) |$. This is the ``Reverse KL'' (RKL) loss, which we approximate with a monte carlo estimate, 
\begin{align*}
KL[& \ p_z(z) \ \| \  p^*_{\theta} (g(z))| \det J_g(z) | \ ]\approxeq 
 \frac{1}{N} \sum_{i=1}^N \log \frac{p_z (z_i)}{ p^*_{\theta} (g(z_i))| \det J_g(z_i) |}\ . 
\end{align*}
In practice our loss function neglects $\log{p_z (z_i)}$ since it does not include parameters of $g$. Further, $\log{ p^*_{\theta} (g(z_i))| \det J_g(z_i) |}$ is only required up to a constant allowing avoidance of intractable normalization terms, e.g., when $p^*(\theta) = e^{u(\theta)}/Z$.
Once trained, $\theta = g(z) $ with  $z\sim p_z$ gives  approximate samples from previously intractable $p_\theta$! 

In this work, we leverage the Planar Flows of Rezende \& Mohamed \cite{rezende2015variational} and their more expressive generalization, Sylvester Flows \cite{berg2018sylvester}; the target is $p_\theta^* = p_{ExpM}$; and base, $p_z$, is Gaussian.


%% file: 31-algorithm.tex

\begin{algorithm}
    \caption{Exp+NF}\label{alg:cap}
    \label{alg: expm+nf}
    \begin{algorithmic}
    \hrule height 0.04cm
    \STATE \textbf{Input:} Training set $(X,Y)$, a base distribution $p_z$, a loss function $L$, decreasing $\phi$, privacy budget $\varepsilon$, sensitivity bound $s$, epochs $T$.
    \STATE \textbf{Initialize} $\beta_{\text{NF}, 0}$ randomly
    \FOR{$t$ in $[T]$ \COMMENT{Train Auxiliary NF Model}}
    \STATE \textbf{Generate sample from base distribution}  
    \STATE $z_i \sim p_z, i = 1, ..., N$ 
    \STATE \textbf{Produce parameter sample \& Jacobian}
    \STATE \indent $\theta_i = g_{\beta_{\text{NF}, t}}(z_i)$ 
    \STATE \indent $\log \ |\det J_g(z_i) |$ (depends on NF) 
    \STATE \textbf{Compute utility for each sample}
    \STATE for each $i \in [N]$, $u(X,Y,\theta_i) = \phi \circ L(X,Y, \theta_i)$
    \STATE \textbf{Evaluate target log likelihood (up to a constant)}
    \STATE $\log[\ p^*_{\theta}(\{\theta_i\}) Z\ ] = \tfrac{\varepsilon}{2s}\sum\limits_{i=1}^N u(X,Y, \theta_i)$
    \STATE \textbf{Compute Reverse KL Loss}
    \STATE $\mathcal{L} = KL[ p_z(z_i)\ \| \   p_{\theta}^* (\theta_i)| \det J_g(z_i) | ] + C$ 
    \STATE{\textbf{Update NF Model}}
    \STATE $\beta_{\text{NF}, t+1} = \beta_{\text{NF}, t} - \eta_t \nabla_{\beta_{\text{NF},t}} \mathcal{L}$ 
\ENDFOR
\STATE sample from base distribution, $z \sim p_z$
\STATE \textbf{Output} Model parameters, $\theta = g_{\beta_{\text{NF}, T}}(z)$
\hrule height 0.04cm
\end{algorithmic}
\end{algorithm} 



%% file: 30-expm.tex
\section{ExpM+NF}
\label{sec:method}
Figure \ref{fig:expm+nf} depicts the setting that we will consider, training and publicly releasing an ML algorithm on private data, and the proposed method, ExpM+NF, as described in Algorithm \ref{alg: expm+nf}. 
For training data $(X,Y)$, we seek  $\theta$ that minimizes a loss function $L(X, Y, \theta)$.
ExpM requires specifying a utility function $u(X, Y,\theta)$ to be maximized in $\theta$; thus, $u = \phi \circ L$ for strictly decreasing $\phi$. 
Throughout, we will consider $u = -L$. 
To define the ExpM density, we must provide $s$, a bound on the sensitivity of $u$ (Defn. \ref{defn:sensitivity}). 
 (See Appendix \ref{sec:appendix-expm} for more details on $\phi$ and $s$). 
In this initial work, we consider $\ell^2$ loss functions ($\sum_{X,Y} (\hat{y}(x) - y)^2$) on classification tasks with model outputs and targets bounded by 1. 

\begin{theorem} [Sensitivity bound for $\ell^2$ loss]
\label{thm:sensitivity}
Suppose our targets and model outputs reside in $[0,1]$.
If
$u = -L = -\sum_{(x,y)} (\hat{y}(x, \theta) - y)^2$, 
then $s \leq 1$.
\end{theorem}

\begin{corollary}
$s\leq 1$ for classification tasks using binary/one--hot  targets and logistic/soft--max outputs with $\ell^2$ loss and $u = -L.$
\end{corollary}

\begin{proof} 
Define the utility to be the negative loss so $u = -L = \sum_{(x,y)\in (X,Y)} l(x,y,\theta)+ r(\theta)$  where $l(x,y,\theta)$ is the loss from data point $(x,y)$, and $r(\theta)$ is a regularization function.
For two neighboring datasets $(X,Y), (X', Y')$, let $(x,y)\in (X,Y)$ and $(x',y') \in (X', Y')$ denote the lone row of values that are not equal. 
We obtain
$|u( X, Y, \theta) -u(X', Y', \theta)| = |l(x,y,\theta)-l(x',y',\theta)|$ (all terms of $L$ cancel except that 
of the differing data point); hence, bounding $s$ is achieved by Lipschitz or $\ell^\infty$ conditions on the per--data--point loss function $l$. 
Regularization $r$ does not affect sensitivity. 

Since $\hat{y}(x, \theta), y \in [0,1]$ we see 
$(\hat{y}(x, \theta) - y)^2\in [0,1]$ and $(\hat{y}(x, \theta) - y')^2\in [0,1]$ so 
$|u( X, Y, \theta) -u(X', Y', \theta)| = |(\hat{y}(x, \theta) - y)^2 - (\hat{y}(x', \theta) - y')^2|\leq 1$.
\end{proof}


Near optimal $\theta$ are found by sampling 
$$\theta \sim p_{ExpM} (\theta) \propto \exp [u(X,Y, \theta)\varepsilon/2s]$$
equipped with a pre-specified $\varepsilon$-DP guarantee.
This step replaces an SGD scheme but fills the same role, producing approximately optimal parameters.
As sampling continuous $\theta \sim p_{ExpM}$ is intractable, we train an auxiliary NF model as discussed in Section \ref{sec:nf}.
NF training is computationally well suited for approximating the ExpM density. 
The NF's loss (Reverse KL ) only requires computing the determinant Jacobian (made tractable by the design of the NF) and $\log [p^*_{\theta}(\theta)Z] =  u\varepsilon/(2s)$, essentially the target model's loss function. 
Finally, if the NF output density is a good approximation of the ExpM density, a sample from the NF can be released with nearly $\varepsilon$-DP guarantee (see Theorem \ref{thm:convergence-dp}).



%% file: 40-experiments.tex
\section{Testing ExpM+NF Model Training on MIMIC-III Data}
\label{sec:experiments}
To test the feasibility of ExpM+NF, we provide experimental results by training a logistic regression (LR) and a deep recurrent network (GRU-D) on MIMIC-III---a large, open, electronic health record database. 
We use the binary prediction task benchmarks ICU Mortality and ICU Length of Stay ($>$ 3 days) from Wang et al. \cite{wang2020mimic}. 
Results of non-private LR and GRU-D \cite{che2018recurrent} models are provided by Wang et al., and results from training via  DPSGD (with RDP accountant) appear in Suriyakumar et al. \cite{suriyakumar2021chasing}.  
To the best of our knowledge, Suriyakumar et al. \cite{suriyakumar2021chasing} present the state of the art in differentially private ML on these MIMIC-III benchmarks.
Unfortunately, their code is not provided. 
We replicate these benchmarks as closely as possible. 

For all tasks and models we implement a non-private (SGD) baseline, and DPSGD- and ExpM+NF-trained versions. 
ExpM+NF is currently constrained to using $\ell^2$ loss (needed for the sensitivity bound), while for SGD and DPSGD both $\ell^2$ and binary cross entropy (BCE) loss are used. 

As $\varepsilon$ decreases, variance of the ExpM distribution increases; intuitively this should lead to a greater likelihood of sampling a model with poor accuracy. 
Differentially private ML research generally considers $\varepsilon\in [1,10]$ \cite{ponomareva2023how}. 
We present results for $\varepsilon \in [1\mathrm{e}{-4}, 10]$.

For each $\varepsilon$, we performed a randomized grid search of hyperparemeters fitting on the train set and evaluating on the development set. 
For DPSGD, we use the PRV method of accounting via Opacus \cite{pytorchOpacus, yousefpour2021opacus}. 
Our NF models use base distribution $N(\vec{0},  \text{diag}(\sigma^2) )$, 
and we treat $\sigma$ as a hyperparameter. 
We use AUC as our accuracy metric.
We present the median AUC on the held-out test set across ten models,  each with different random seeds, to provide robustness to outliers. 
For ExpM+NF training, we also train ten NF models (each is a density over our target classifier).  
We use the median AUC of 1K samples per each of the ten NF models. Box plots and per-NF AUC densities are plotted in the Appendix Figure \ref{fig:distribtions_grud_icu} for GRU-D ICU (and analogous plots for all experiments are in the code base, URL on front page).

As all our models are implemented in PyTorch, we use PyTorch's benchmark suite which (only) reports the mean timing of $n$ runs, and only times the GPU training time (no data IO, etc.). 
We use $n = 10$ runs (per task, model, training method, $\varepsilon$) on a single NVIDIA A100 80GB GPU and 128 CPUs with 2TB of shared RAM.
For DPSGD and ExpM+NF, we report the mean across all $\varepsilon$. 
For the non-private GRU-D implementation, we report timing with batch normalization which was fastest. 

Further details, including  information on the MIMIC-III data and benchmarks, our hyperparameter search, and how we present the results are in the Appendix \ref{sec:appendix-experiment-details}.




%% file: 43-bigfig.tex
\begin{table*}
\centering
\caption{Median AUC results on MIMIC-III Mortality and Length of Stay prediction experiments}
\label{tab:mimic_los_mort_results}
\begin{tabular}{lcc}
\toprule
    \rotatebox[origin=l]{90}{\textbf{\phantom{aaaaaaaaaa} Mortality, LR} }
        &  \includegraphics[width=0.4\linewidth]{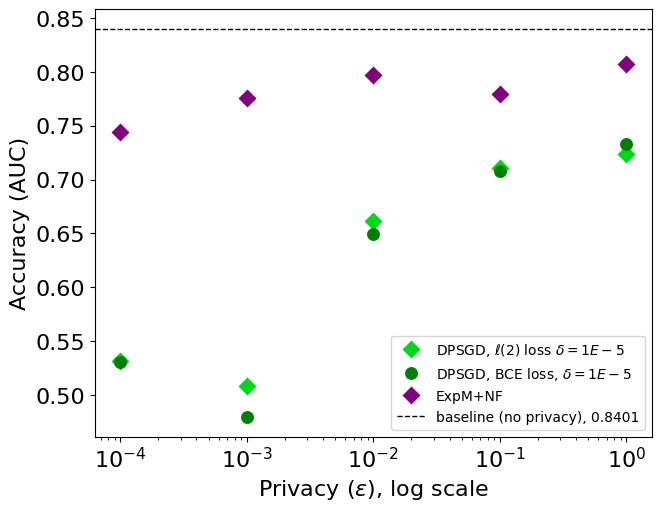} 
        &  \includegraphics[width=0.4\linewidth]{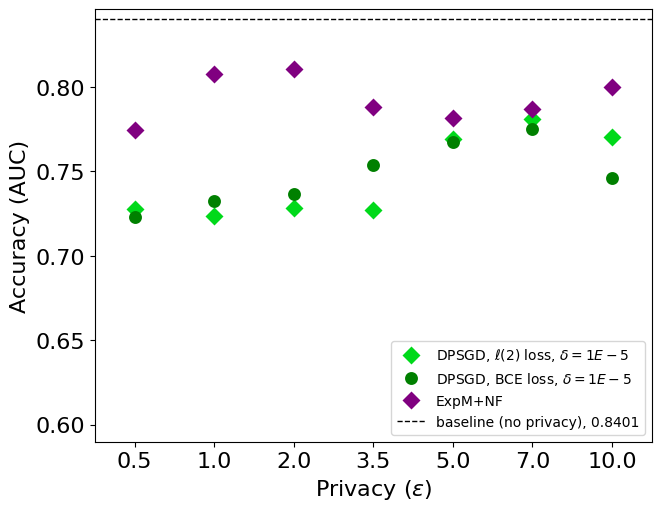}
       \\ 
    \rotatebox[origin=l]{90}{\textbf{\phantom{aaaaaaaa} Mortality, GRU-D} }
        &  \includegraphics[width=0.4\linewidth]{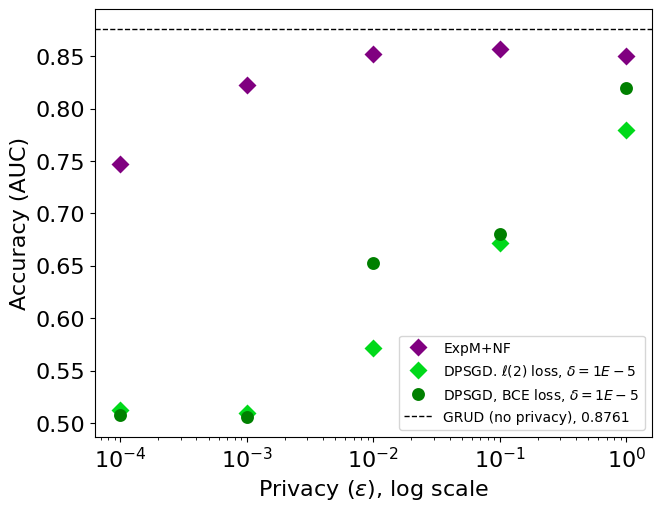}
        & \includegraphics[width=0.4\linewidth]{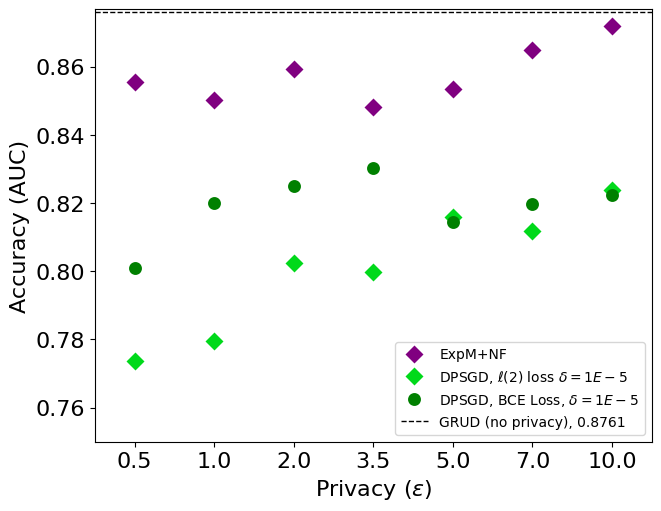}
        \\
\cmidrule(lr){1-1}  \cmidrule(lr){2-2}              \cmidrule(lr){3-3}
    \rotatebox[origin=l]{90}{\textbf{\phantom{aaaaaaaa} Length of Stay, LR} }
        &  \includegraphics[width=0.4\linewidth]{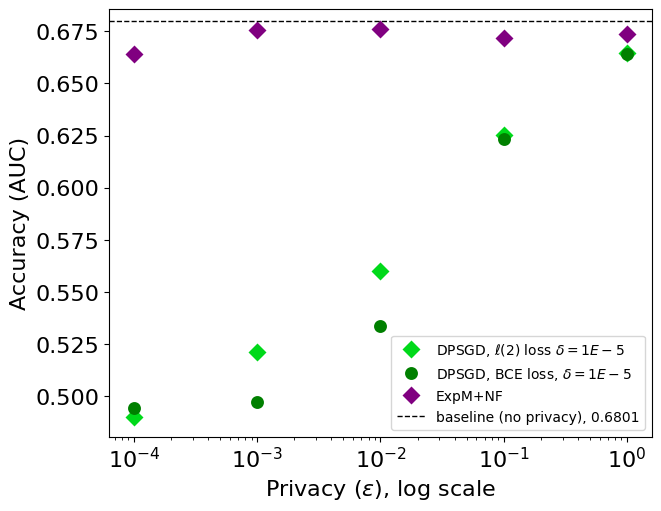} 
        &  \includegraphics[width=0.4\linewidth]{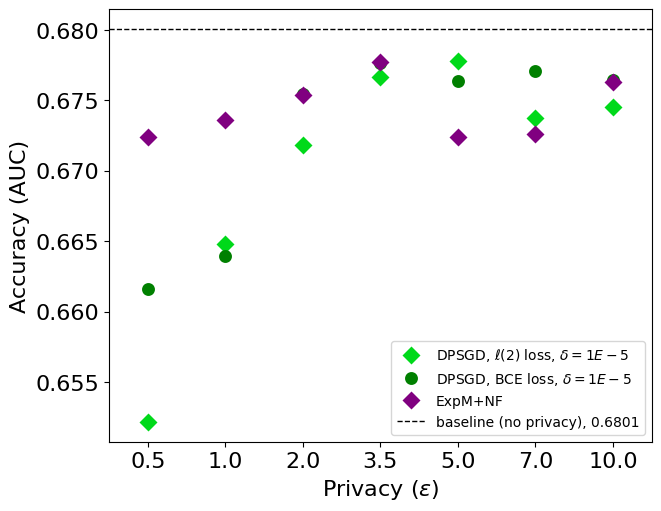}
       \\ 
    \rotatebox[origin=l]{90}{\textbf{\phantom{aaaaa} Length of Stay, GRU-D} }
        &  \includegraphics[width=0.4\linewidth]{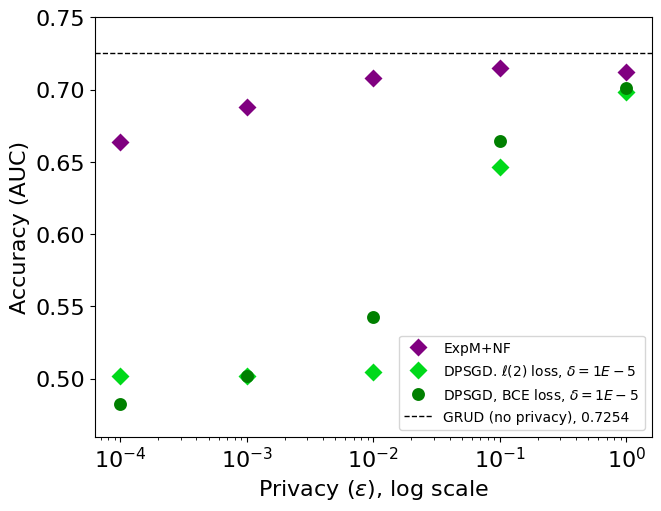}
        & \includegraphics[width=0.4\linewidth]{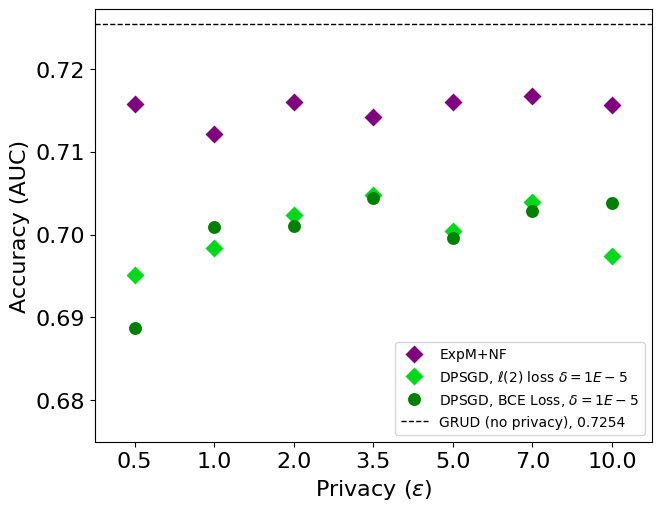}
        \\
\bottomrule
\end{tabular} 
\end{table*}

%% file: 44-results.tex
\subsection{Accuracy and Timing Results}
\label{sec:accuracy-results}
\label{sec:times}

The table of figures \ref{tab:mimic_los_mort_results} displays accuracy as we vary $\varepsilon$. 
In every experiment, ExpM+NF  exhibits greater than 94\% of the non-private models' AUC for all $\varepsilon\geq 1\mathrm{e}-3$. 
ExpM+NF surpasses DPSGD in all experiments except $\varepsilon \geq 2$ with the LR model on the Length of Stay task, in which all privacy methods achieve 99\% of the non-private model's AUC. 
Viewed through another lens, ExpM+NF maintains a greater AUC for $\varepsilon \geq 1\mathrm{e}-3$ than DPSGD for $\varepsilon = 1, \delta = 1\mathrm{e}-5$ in all MIMIC-III experiments, respectively. 
These results show that ExpM+NF is a viable training method even for $\varepsilon$ values orders of magnitude smaller than is currently considered. 
\textit{We do not claim ExpM+NF achieves $\varepsilon$-DP as this is not yet proven.}
Appendix \ref{sec:appendix-more-results} provides other findings including  NF distributions and results showing our DPSGD implementations push the previous state of the art for these MIMIC-III benchmarks. 

More discussion and takeaways of these results are provided in the Conclusion \ref{sec:conclusion}.

Timing results, which report only the computational time due to training, appear in Table \ref{tab:times} for ICU Mortality experiments. 
Timing results for Length of Stay are similar and reside in Appendix \ref{sec:appendix-times} along with more details of this timing experiment.
Non-private (SGD) is the fastest, as expected). 
When including computation of the \texttt{noise_multiplier}, DPSGD is slower than ExpM+NF for both LR and GRU-D training. 
The innovation here is, when computing the RKL loss for NF training, our ExpM+NF code computes (and tracks gradients through) the $\ell^2$ loss for \textit{all} the sampled parameters $\theta_i = g(z_i)$ in parallel, rather than computing this sequentially. 
We used the best hyperparmeters for each model, so batch size and epochs vary between the different tasks and $\varepsilon$ levels. 

\begin{table*}[ht]
\centering
\caption{ICU Mortality Mean Benchmark Timing Results}
\label{tab:times}
\begin{tabular}{cc ccc ccc}
\toprule
& & \multicolumn{3}{c}{Logistic Regression} & \multicolumn{3}{c}{GRU-D} \\
        \cmidrule(lr){3-5} \cmidrule{6-8}
\multirow{2}{*}{Training} & \multirow{2}{*}{Loss} & Training & Computing & \multirow{2}{*}{Total} & Training & Computing & \multirow{2}{*}{Total} \\
    &   &  Time & $\varepsilon$ & &  Time & $\varepsilon$ & \\
\cmidrule(lr){1-2}              \cmidrule(lr){3-5} \cmidrule{6-8}
\multirow{2}{4em}{Non-Private} & $\ell^2$ & .87 ms & $-$ & .87 ms & .84 ms & $-$ & .84 ms\\
 & BCE & 1.67 ms & $-$ & 1.67 ms & .93 ms & $-$ & .93 ms\\
 \cmidrule(lr){1-2} \cmidrule(lr){3-5} \cmidrule{6-8}
 \multirow{2}{4em}{DPSGD} & $\ell^2$ & 1.16 ms & 1.13 ms & 2.29 ms & 1.38 ms & 1.44 ms & 2.82 ms\\
 & BCE & 1.56 ms & 1.2 ms & 2.76 ms & 1.46 ms & 1.34 ms & 2.8 ms\\ 
  \cmidrule(lr){1-2} \cmidrule(lr){3-5} \cmidrule{6-8}
ExpM+NF & RKL+$\ell^2$ & 1.65 ms & $-$ & 1.65 ms & 1.71 ms & $-$ & 1.71 ms
 
\\
\bottomrule
\end{tabular}
\end{table*}

%% file: 50-uq.tex
\section{ExpM+NF for Bayesian Inference \& Uncertainty Quantification (UQ)}
\label{sec:uq}
Privacy diminishes with each sample of ExpM. 
In the private scenario (as our experiment above supposed) we consider the frequentist approach of finding a near optimal $\hat{\theta}$ (MLE or MAP) via a \textit{single sample} from ExpM+NF, and, given input $x$, predict $\mathbb{E}_y(y|x, \hat{\theta})$. 
Without these constraints, the rapid sampling capability of the ExpM+NF allows (1) MAP estimates by identifying the maximum utility sample and (2) facilitates Bayesian inference\textemdash
if $p(y|x,\theta) \propto \exp(-\varepsilon l(x,y,\theta)/(2s))$ and the prior $p(\theta)= \exp(-r(\theta))$ for regularization function $r$, then $p_{ExpM}(\theta| X, Y )$ is the posterior, from which ExpM+NF allows us to approximately sample $\theta_i \sim p_{NF}(\theta|X,y), i = 1, ..., N$, rapidly. 
Now one can: perform inference, given input $x$,  predict $\mathbb{E}_y(y|x) = E_\theta(p(y=1|x, \theta))\approx (1/N) \sum_i \exp(-\epsilon l(x,y,t_i)/2)$, the prediction encompassing the uncertainty in $\theta$; and compute credibility regions on $\theta$ and on $y$. Details and an example reside in Appendix \ref{sec:appendix-uq}.

%% file: 60-mia.tex
\section{Empirically Privacy Experiments}
Just as measuring accuracy on a hold-out set can be a good proxy for generalization error,
we present two experiments to investigate the empirical measures of privacy of ExpM+NF to gauge actual privacy risk. 

\subsection{Membership Inference Attack}
\label{sec:mia}
In order to determine if ExpM+NF exhibits empirical evidence of privacy, we utilize the Likelihood Ratio Attack (LiRA) (Algorithm 1  of \cite{carlini2022membership}), the state--of--the--art  membership inference attack (MIA), to attack models trained with SGD (non-private), DPSGD, and ExpM+NF. 
Given a target model $\hat{M}$ and data point $x_0$, LiRA seeks to determine whether $x_0$ is in the training set. 
To do so, one trains ``shadow'' models, some with $x_0$ in the training data ($M_i^{in}$) and some that exclude $x_0$ from the training data ($M_i^{out}$). 
Using the shadow models, the confidence  $M_i^*(x_0)\in [0,1]$ is computed, then a logit scaling is applied giving values $\operatorname{logit}( M_i^*(x_0)) \in \R$ that are Gaussian distributed. 
The means and variances of the in/out sets are recorded. 
Finally, the likelihood of  $\operatorname{logit}(\hat{M}(x_0))$ 
is computed under both Gaussians to see if $\hat{M}$ is more likely trained with or without $x_0$.


For our experiments, we train 1000 non-private shadow models using the best hyperparameters found from the grid search performed for the results in Section \ref{sec:experiments}. For implementation efficiency, rather than exclude a single data point from the training set to train a shadow model, we randomly exclude half the training data so that any one data point is included in the training data for approximately half of all shadow models. 
Q-Q plots are used to confirm the logit-scaled confidences are Gaussian. 
We then train 50 target models on a random half of the training data, for each method (non-private, DPSGD, ExpM+NF) and each model type (LR, GRUD), again using the best hyperparameters found in section \ref{sec:experiments}. 
As discussed by Carlini et al. \cite{carlini2022membership}, 
the low-false-positive region is most important for MIAs,  so we plot the median ROC across the 50 target models on a log-log scale, with 5-95\% quantiles shaded. 
We consider an attack unsuccessful if, in the low-false-positive region, the true-positive rate is at most the false-positive rate, meaning the ROC curve in this region is approximately the curve $y=x$ or below it.

\subsubsection{MIA Results}
Figure \ref{fig:qqplot} show a histogram of two sampled data point's logit scaled confidences when included or excluded from a shadow model's training set along with corresponding Q-Q Plots, showing the confidences are nearly Gaussian. 
As the two Gaussians are often nearly identical, accurate membership inference will be difficult. 

\begin{figure*}
    \centering
    \includegraphics[width=.47\textwidth]{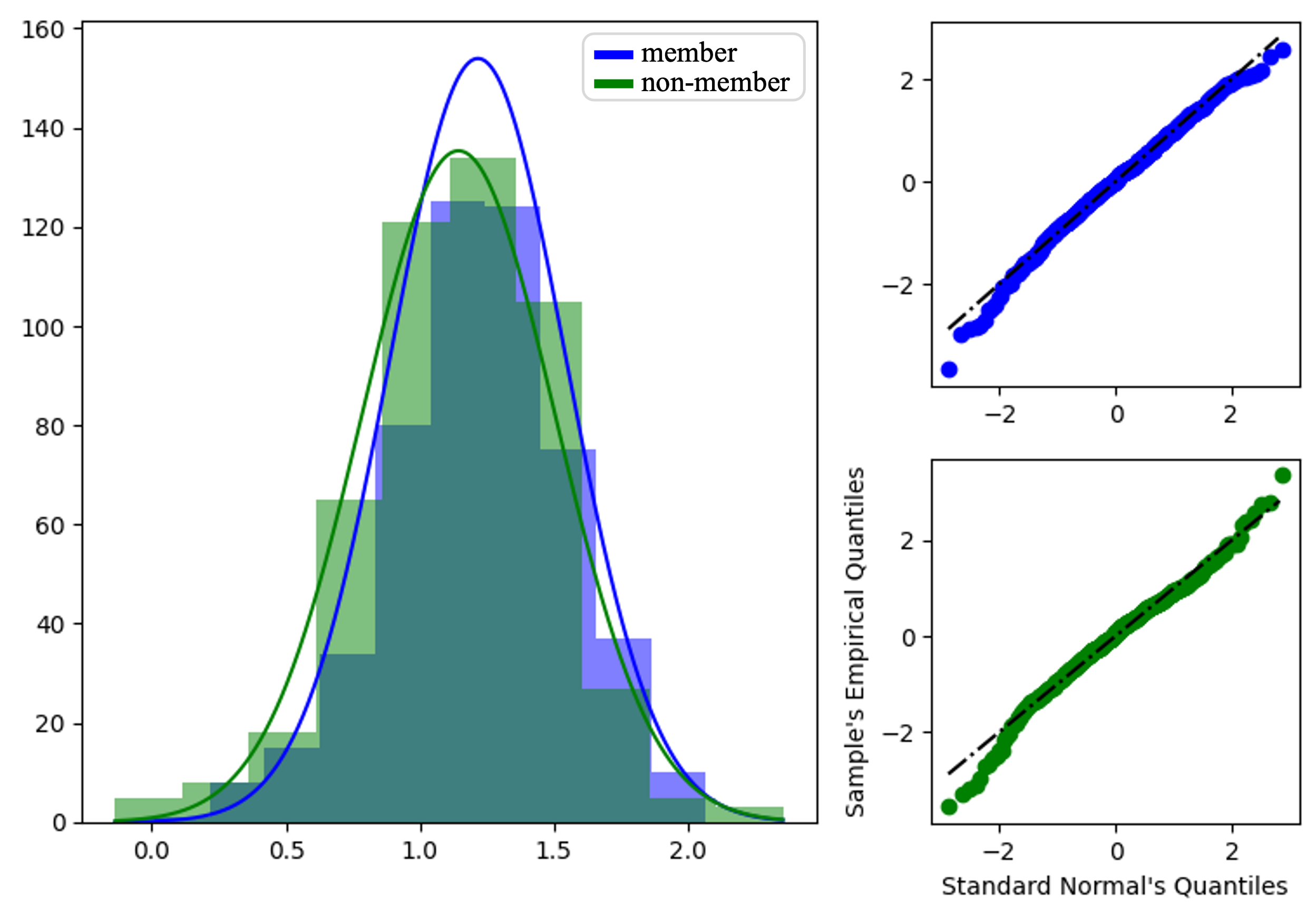}
    \includegraphics[width=.47\textwidth]{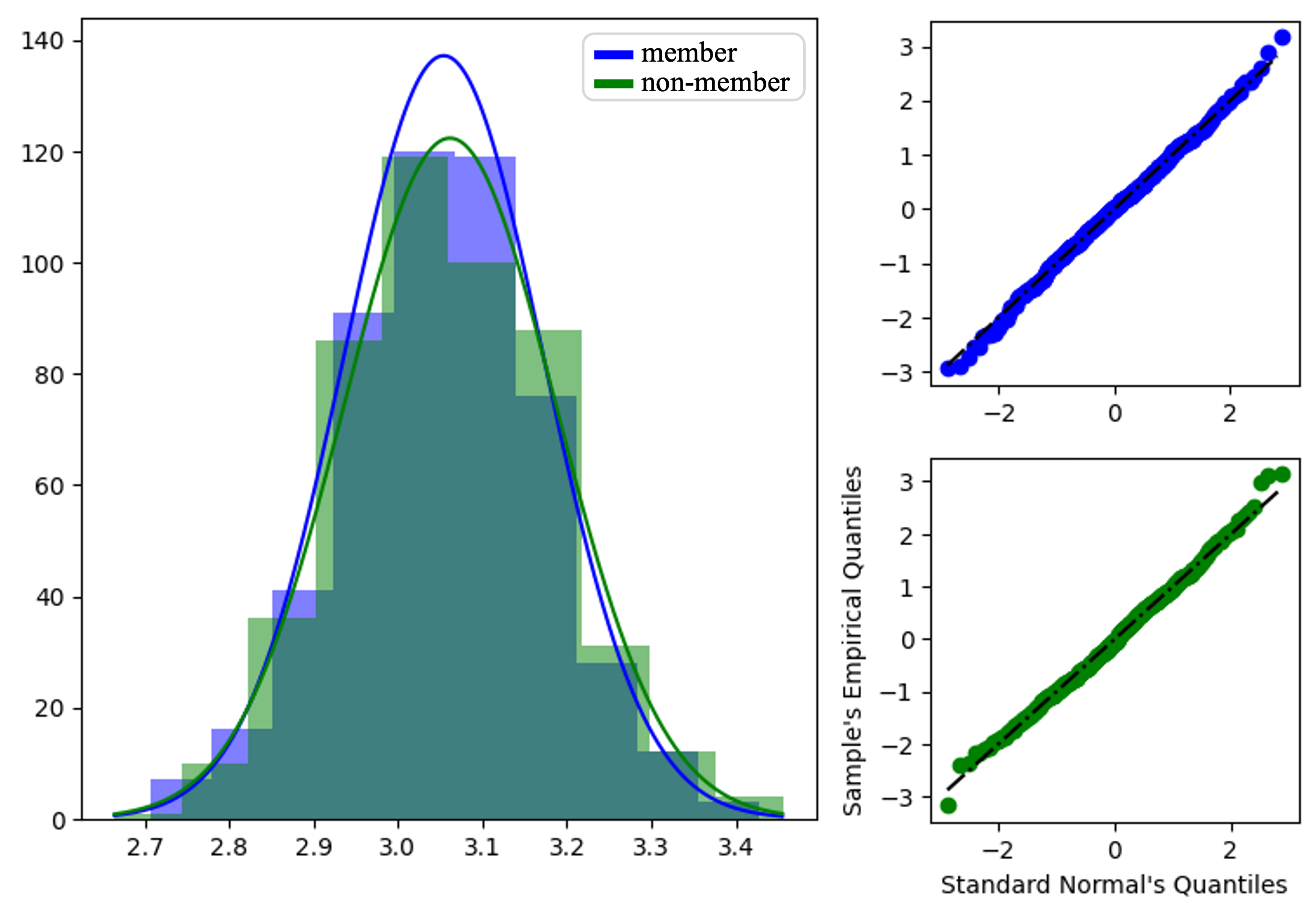}
    \caption{Both the left and the right plots depict histogram and Q-Q plots for a random data point. The blue are the logit scaled confidences when the data point is included in the shadow model's training set and the green are the logit scaled confidences when the data point is not a member. These show that the per-example logit scaled confidences are nearly Gaussian.}
    \label{fig:qqplot}
    \vspace{-.25cm}
\end{figure*}

Figure \ref{fig:lr_mort_icu_mia_auc} shows ROC curves for 9 different $\varepsilon$ values. 
ExpM+NF is unsuccessfully attacked for every $\varepsilon$ except potentially $\varepsilon = 2$ where ExpM+NF is slightly above linear. 
DPSGD is similarly, except in the $\varepsilon = 0.0001$ and $\varepsilon = 0.1$ case, where LiRA is worse than random chance. 
In all cases, ExpM+NF and DPSGD are more resistant to LiRA than non-private SGD. 

\begin{figure*}
    \centering
    \includegraphics[width=\textwidth]{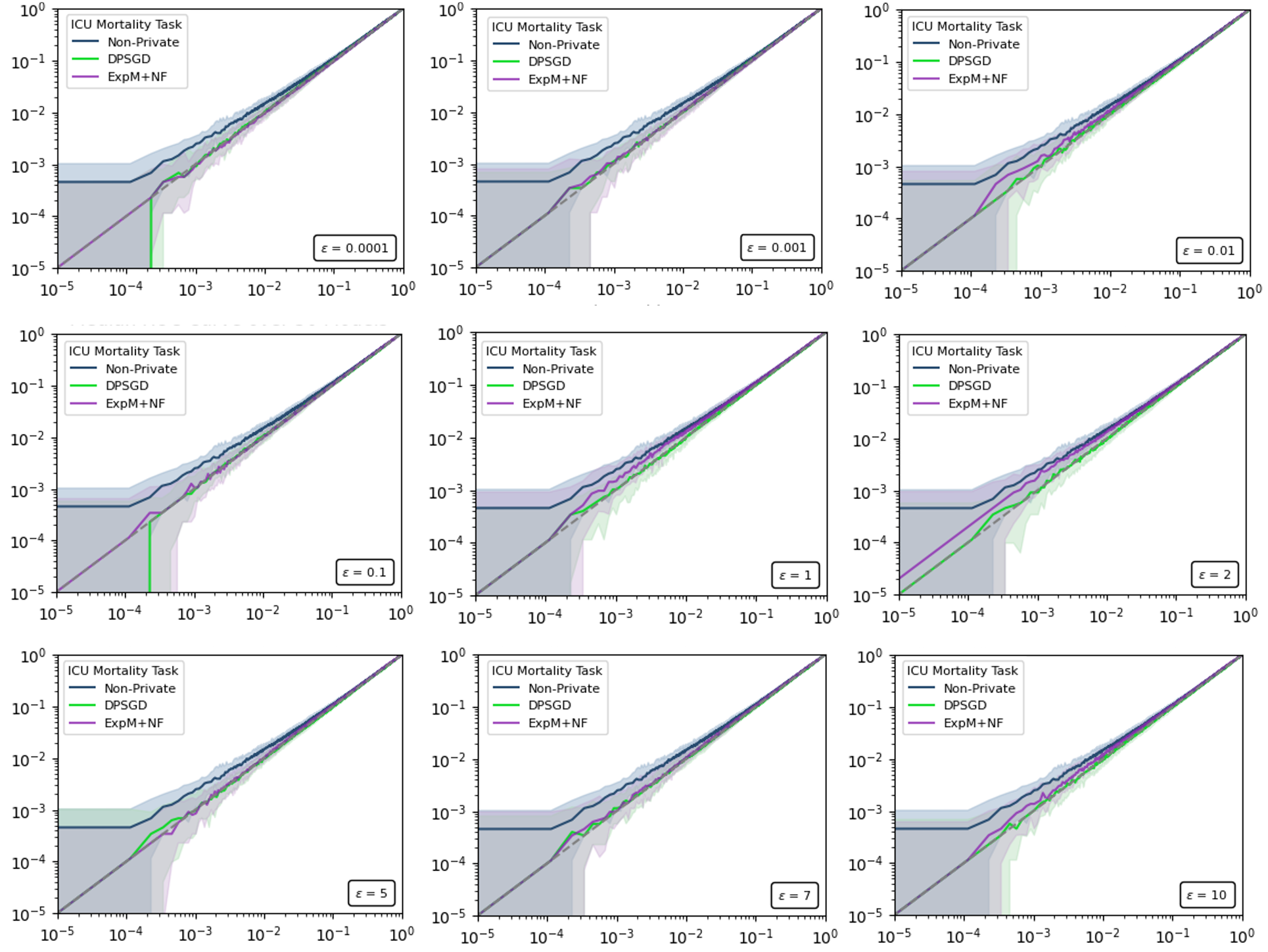}
    \caption{Likelihood Ratio Attack (median) ROC curve with 5th and 95th percentiles shaded on a log-log for Logistic Regression on Mortality Task.}
    \label{fig:lr_mort_icu_mia_auc}
    \vspace{-.25cm}
\end{figure*}

We see in Figure \ref{fig:lr_los_3_mia_auc}, that ExpM+NF is only unsuccessfully attacked when $\varepsilon = 0.0001$, and  
in every other case, is slightly better than non-private SGD. 
DPSGD on the other hand is unsuccessfully attacked for every $\varepsilon$. 
Hence, ExpM+NF empirically provides more privacy than non-private SGD but less than DPSGD on this task, model. 

We provide plots for GRU-D models in the Appendix \ref{sec:appendix-lira}. 
Similar to logistic models on the Mortality task, GRU-D models trained with ExpM+NF are unsuccessfully attacked for all $\varepsilon$ across both Mortality and Length of Stay with the exception of the $\varepsilon = 0.0001$ on the Length of Stay task, where ExpM+NF does worse than random guessing. 
On the other hand, DPSGD is successfully attacked for $\varepsilon = 2$ and $\varepsilon = 7$ on the Length of Stay task, indicating that that even models with a differential privacy guarantee can be successfully attacked. Interestingly, models trained non-privately are entirely unsuccessfully attacked. 
This indicates that GRU-D trained with no privacy exhibits emperical privacy, and, hence, it may be difficult to interpret the empirical privacy from these results. 
We consider the GRU-D LiRA results to be inconclusive. 

\begin{figure*}
    \centering
    \includegraphics[width=\textwidth]{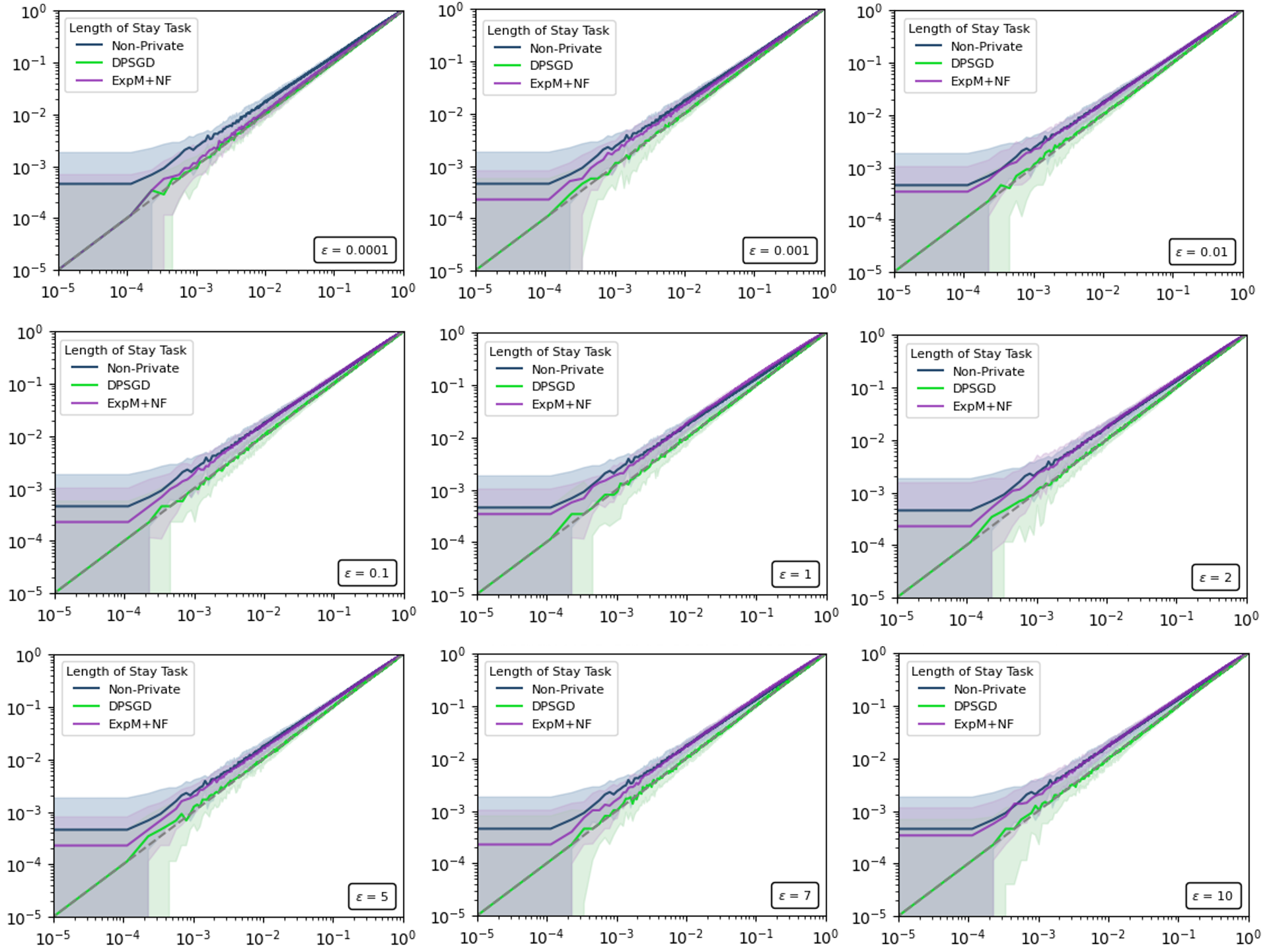}
    \caption{Likelihood Ratio Attack (median) ROC curve with 5th and 95th percentiles shaded on a log-log for Logistic Regression on Length of Stay Task.}
    \label{fig:lr_los_3_mia_auc}
    \vspace{-.25cm}
\end{figure*}

Unlike Carlini et al.\ who were able to see true-positive rates of $0.1$ at false-positive rate of $10^{-5}$, we are at best able to get a true-positive rate of $10^{-3}$ at that same false-positive rate when attacking a model trained non-privately. 
Additionally, outside of the low-false-positive region, our overall AUC is nearly random guessing across all methods and model types. There may be several reasons that the membership inference attack was less successful in our situation.
The first may be due to the number of parameters in the models we attack. Carlini et al.\ ``find that for both the CNN and WRN model families, larger models are more vulnerable to attacks.'' 
While none of the models we attack fall within these model families, the number of parameters in the models we attack are on the order of thousands to tens of thousands whereas in \cite{carlini2022membership}, the models likely have on the order of millions or tens of millions of parameters. Additionally, MIMIC III is a data set created from real health data and because of that certain privacy preserving methods were applied to the data prior to release that could have potentially led our attack to be less successful.





\subsection{Privacy Auditing}
Recent work of Steinke et al. \cite{steinke2024privacy} proves a lower bound on the privacy $\varepsilon$, that holds with known likelihood, and is computable from the loss with just black-box access.
The intuition is that if an auditor can correctly guess membership of ``canary examples'', then almost certainly the model's privacy is weak. 
Here we apply this recent privacy auditing methodology. 

Using notation of Steinke et al., we let $m = |X|$ and use a uniform randomly chosen half of the data (the ``in'' set) for training a (single) model.  
Storing initially randomized parameters $\omega_0$ and trained parameters $\omega_1$, we 
use score function $x\mapsto l(x, \omega^0) - l(x, \omega^1)$, where $l$ is the point-wise loss function. 
Intuitively, $x$ is in the training set if it has relatively low/high trained loss $l(x, \omega^1)$, equivalently, a high/low score, respectively. 
Sorting the scores, we label the top $k_{+} = m/2$ scored data points as $1$ (in), bottom $k_{-} = m/2$  as $-1$ (out), and record $v$, the number of correct labels by this audit. 
Setting p-value $\beta$  and a given $\delta$, Steinke et al. \cite{steinke2024privacy} (see their Appendix D) uses the number of correct/incorrect guesses and furnishes the computable lower-bound, 
$\varepsilon^{lb}$ on the true privacy, $\varepsilon$, that holds with likelihood $1-\beta$.

We set $\delta = 0$ and then $\delta = 10^{-5}$, and use p-value $\beta = .05.$
For each training method, (non-private SGD, DPSGD, ExpM+NF), we repeat the experiment 50 times to obtain and sort lower bounds $\varepsilon^{lb}_k, k = 1, ..., 50$. 
Since each has 95\% confidence of being a lower bound, we report the 95-th percentile ($\varepsilon^{lb}_{k=48}$ third highest value).  
To implement the lower bound computations, we use code provided by Steinke et al. 

\subsubsection{Privacy Audit Results}
We found for every method (non-private SGD, DPSGD, ExpM+NF) and every model type (LR, GRU-D) across both benchmark tasks (Mortality and Length of Stay)  that $\bm{\varepsilon}^{\bm{lb}} \mathbf{= 0}$, irrelevant of the $\tilde \varepsilon$ value used to train DPSGD or ExpM+NF. 
These unhelpful results are evidence that the MIAs against all three training methods are impotent. 
As one of the first deployments of this new, exciting auditing, our results give a glimpse at the limitations met in practice. 

%% file: 70-math.tex
\section{Towards ExpM+NF Privacy Proof}
\label{sec:math}
For this section we will fix the loss / utility function $u$ and privacy parameter ${\varepsilon}$. 
Let $\P$ denote the set of probability densities on $\R^n$. 
We consider the Exponential Mechanism $p: D\to \P$, i.e., for a dataset $X\in D$ we let $p_X(t) = p(X,t)$ denote the ExpM density. 
Similarly for ExpM+NF, we let $q: D\to  \P$ with $q_X(t) = q(X,t)$ denoting the trained NF density approximating $p_X.$
Now, we know that $\log  [ p_X/p_X' ] \leq {\varepsilon}$ for all neighboring $X, X'\in D$. 
Our goal is to show ExpM+NF attains $h(\varepsilon)$-DP for some function $h$; hence, we seek to prove that $\log  [ q_X/q_X' ] \leq h(\varepsilon)$. 
Although we have not proven privacy of ExpM+NF, this section has two goals. 
The first is to document obstructions to proving ExpM+NF privacy. 
The second is to present mathematical results that seem generally useful to the community. 

Direct proof techniques of ExpM+NF privacy are problematic. 
Tracking the maximal possible changes of the parameters of the NF, say $\beta_X, \beta_X'$, through each gradient step leads to compounding differences, even when ignoring different random starting positions of the parameters. 
Further, one must translate the NF parameters $\beta_X, \beta_X'$ into the desired estimate $|q_X/q_X'|$. This seems problematic as many different NF instantiations can lead to similar, even identical output densities; e.g., the three layer NF $z\mapsto 3z \mapsto z/3$ provides the identity map, which is clearly same as the NF $z\mapsto z \mapsto z$, the identity  in each layer. 
Instead, in Section \ref{sec:math-convergence-results}, we consider the more general case of a sequence of densities $q_n(X, t)$ converging to density $p(X,t)$  which satisfies $\varepsilon$-DP, and seek mathematical developments  furnishing  privacy of $q_n(X, t)$ from $p(X,t)$. 
In our specific case of ExpM+NF, this converts the problem to proving a convergence theorem for the NF approximation. 
Finally, in Section \ref{sec:squeeze-results}, we provide theorems to furnish privacy of a mechanism $c$ that lie between two mechanisms $a, b$ that have a known privacy guarantee, in the sense of pointwise inequality for all datasets $X.$ 
Intuition for why this is potentially useful for ExpM analysis is given. 
In each subsection, we include more details of mapping these theorems to our ExpM+NF goals after the mathematical developments. 

\subsection{Results on Convergence of Differentially Private Mechanisms}
\label{sec:math-convergence-results}
The results in  this section provide hypotheses giving privacy of an approximation $q$ of the intractable density $p$ for which a privacy guarantee is known. 

\begin{definition}[DP Distance on $\P$]
$d_{DP}(p,q) = \| \log(p)-\log(q)\|_\infty$
\end{definition}
DP distance is a metric since: $d_{DP}\geq 0$;  $d_{DP}$ satisfies the triangle inequality because $\|\cdot\|_\infty$ is a norm; and   $\log$ is 1-1, which gives  $d_{DP}(p,q) = 0 \iff p=q $ almost everywhere. 

The following proposition shows $d_{DP}$  provides a strictly stronger topology than 
$L^{\infty}$. 
More toward our aim, 
Theorem $\ref{thm:convergence-dp}$ shows that convergence in $d_{DP}$  preserves $\varepsilon$-DP, whereas Theorem \ref{thm:uniform-convergence} shows that uniform convergence (in $L^\infty$) only guarantees ($\varepsilon, \delta>0$)-DP, which shows the need for this definition of distance.

\begin{proposition}\label{prop:d_dp}
    Convergence in $d_{DP}$ implies uniform convergence, but not conversely. 
\end{proposition}
\begin{proof}
    Suppose $q_n\xrightarrow[]{d_{DP}} p$. This means $\sup_t |\log(q_n(t)) - \log(p(t))| \to 0$. Since $\exp(\cdot)$ is continuous, we see $\sup_t |q_n(t)-p(t)|\to 0$. 
    To see the converse is false, let 
    \begin{equation}
    \label{eqn:qn}
    q_n = \left\{\begin{array}{cc}
         \tfrac{1}{n} & \text{on } [0, 1/2) \\
         2 - \tfrac{1}{n}& \text{on } [1/2, 1] 
    \end{array}\right.
    \end{equation}
    To see $q_n$ converges uniformly, compute $|q_n - q_m| = |\tfrac{1}{n}-\tfrac{1}{m}|$, showing the sequence is uniformly Cauchy.
    To see $q_n$ does not converge in DP distance, compute
    $d_{DP}(q_n, q_m) = \sup_t |\log(q_n/q_m)| = |\log(m) - \log(n)|$, not Cauchy. 
\end{proof}

\begin{theorem}
\label{thm:convergence-dp}
Suppose $p$ satisfies $\varepsilon$--DP and 
$d_{DP}(p_X,q_X) \leq \eta$ for all $X\in D$. 
Then $q$ satisfies $(\varepsilon + 2\eta)$--DP. 
\end{theorem}
\begin{proof}
By hypothesis, $e^{-\eta}\leq \tfrac{q_X}{p_X} \leq e^{\eta}$ for all $X\in D$.
For neighboring $X,X'\in D,$
$ \tfrac{q_X}{q_X'} =  \tfrac{q_X}{p_X} \tfrac{p_X}{p_X'} \tfrac{p_X'}{q_X'} \leq e^{\eta}e^{\varepsilon}e^{\eta}.$  
\end{proof}


\begin{theorem}\label{thm:uniform-convergence}
    Suppose mechanism $p$ satisfies $\varepsilon$-DP and the sequence of mechanisms $q_n(X,t)\to p(X,t)$ uniformly in $X, t$ as $n\nearrow \infty$ (here we have convergence in $L^\infty$). 
    Then for all $\delta\in (0,1), \eta > 0, \exists N>0$ so that $n\geq N $ implies $  q_n$ is $(\varepsilon+2\eta, \delta)$-DP. 
\end{theorem}
\begin{proof}
Since $\log(t) \leq t-1$ and $\log(t) \approx t-1$ near $t = 1$, 
$$d_{DP}(p_n,p) = \left\|\log\left(\frac{p_n}{p}\right)\right\|_\infty \leq \left\|\tfrac{p_n}{p} -1\right\|_\infty =  \left\|\tfrac{p_n-p}{p} \right\|_\infty .$$
Choose $r>0$ small enough so that $P_X(\{p<r\})\leq \delta$.
Then on $\{p<r\}^c = \{p\geq r\}$ we have $d_{DP}(p_n,p) \leq \|p-p_n\|_\infty/r.$
Now choosing $N$ large enough so that for all $n>N, \|p-p_n\|_\infty < \varepsilon r$ gives that $d_{DP}(p_n, p) \leq \varepsilon$ except on a set  with likelihood $\delta$. 
The result follows from Theorem \ref{thm:convergence-dp}.
\end{proof}

\subsection{Mapping Convergence Results to ExpM+NF}
\label{sec:mapping-convergence-results}
The hypotheses for the convergence theorems must hold for all datasets $X\in D$ (not just the specific dataset in hand). 
Hence, performing diagnostics to estimate the accuracy of the NF, $q_n(X, t)$ in approximating ExpM, $p(X,t))$ for a specific  dataset $X$ is insufficient. 
The primary obstruction to employing the convergence theorems above for ExpM+NF is proving a bound on the proximity of $q_X, p_X$ independent of $X$. 

\subsubsection{An Illustrative Example}\label{sec:example}
\begin{figure}
    \centering
    \includegraphics[width = 0.235\textwidth]{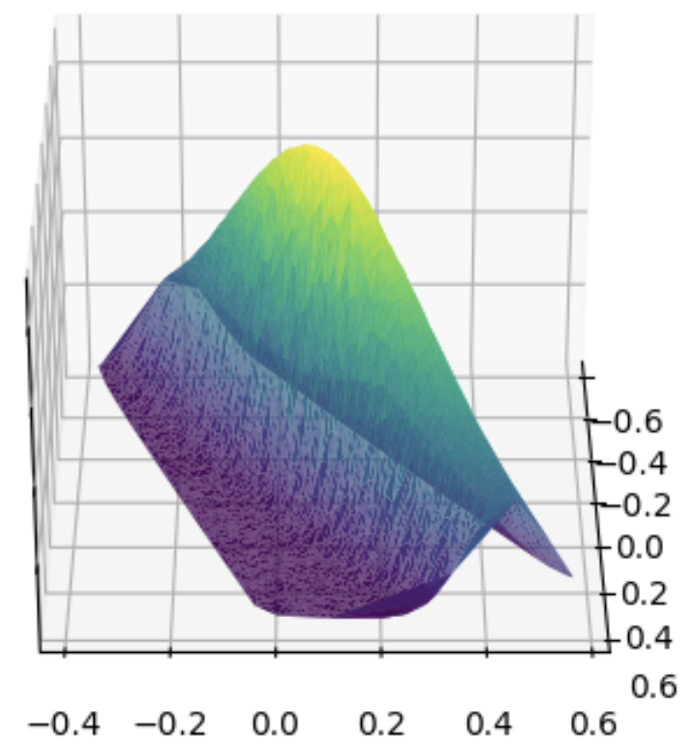}
    \includegraphics[width = 0.235\textwidth]{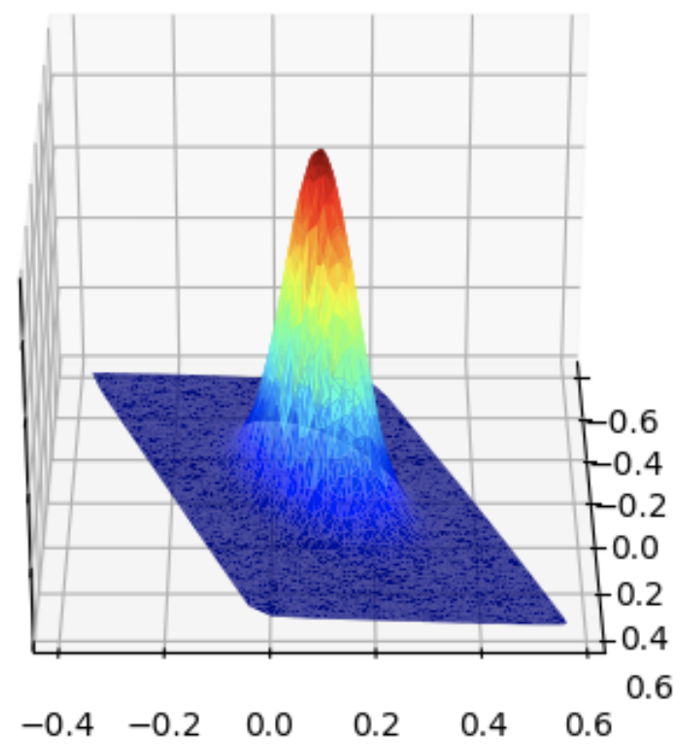}
    \caption{\textbf{Left density} is the target ExpM density, $p_X$ for training a logistic regression classifier with two parameters. \textbf{Right density} is the NF approximation density, $q_X$ of $p_x$. Comparison shows the NF approximation is too precise (too peaked at the mode) to be a good approximation. Hence, sampling from the NF approximation provides accurate model parameters with high likelihood, but intuitively should have strictly less privacy than the target ExpM density, as it has too little variance. Compare with Figure \ref{fig:expm_vary_eps}.}
    \label{fig:counter-example}
\end{figure}
When using the ExpM+NF for training a model $M(t)$, previous research shows the model parameters  $t$  sampled from the NF output density $q$, provide very accurate models. 
This implies that the region near the mode(s) of the ExpM density, $p$, are preserved by the the NF approximation, $q$. 
(By design the mode of ExpM is the $\argmin_t L$, so sampling from the ExpM desity produces near optimal parameters with high likelihood).
Hence, if the NF approximation $q$ is not close to $p$, it is intuitively too precise (too much mass near the mode, too little toward the tail). 
Figure \ref{fig:counter-example} shows this is precisely the case for a specific example. 
To exhibit this example, a logistic regression classifier with two parameters $t = [t_1, t_2]^t$ is trained using ExpM+NF with utility function defined to be  the negative of $\ell(2)$ loss; hence, the ExpM density $p_X(t)$, a density over $\R^2$, is approximated via training a normalizing flow model to produce output density $q_X$. 
Both $p_X$ and $q_X$ are pictured in the figure.  
Comparing this to Figure \ref{fig:expm_vary_eps}, it seems that our NF output density is closer to the ExpM density with worse privacy (less variance). 
This counter example shows that in general we cannot expect $d_{DP}(q, p)\to 0$; hence to apply Theorem \ref{thm:convergence-dp}, a method of bounding  $d_{DP}(q, p)$ is needed. 

From an engineering perspective, this example suggests that regularization of the Reverse KL loss  is needed to dampen the over-precise fitting of the NF.

\subsubsection{Does a general converse to Theorem \ref{thm:convergence-dp} exist? (No)} \label{sec:converse}
If a general converse exists, then it would imply that the any density family $\{q_X\}_{X\in D}$ satisfying close to $\varepsilon$-DP must be uniformly close to the ExpM family $\{p_X\}$. 
Without more hypotheses, we can show such a converse false quite easily. 
Let $q_X$ be the standard normal distribution on $\R^n$ for all $X.$ Clearly $q_X/q_X' = 1$, so $q$ satisfies $0$-DP. 
But any ExpM $p$ satisfying nearly $0$-DP will have to be a uniform density, which is far from the Gaussian density. 
Mapping this back to our goal, without more hypotheses, the NF output density may indeed satisfy a strong privacy guarantee even if it fails to uniformly approximate the target ExpM density.

\begin{figure}
    \centering
    \includegraphics[width = 0.5\textwidth]{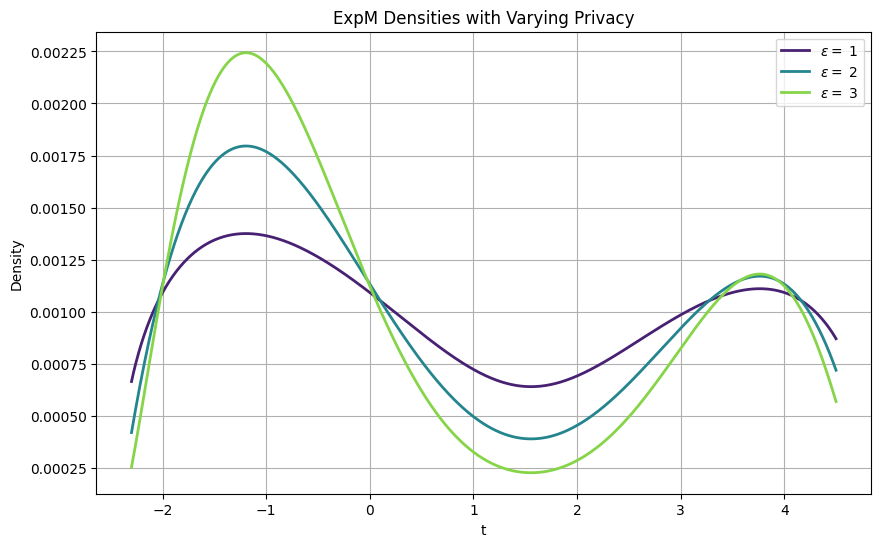}
    \caption{For a given utility function and data set, the variance ExpM density $p_X$ increases as privacy increases ($\varepsilon$ decreases).}
    \label{fig:expm_vary_eps}
\end{figure}

\subsection{DP Squeeze Theorem}\label{sec:squeeze-results}
In this section we let $\Delta(u)$ denote the sensitivity of a function $u(X,t): D\times \R^n\to \R$. 
We denote privacy mechanisms $X\mapsto f(X,t) = f_X(t): D\to  \P$. 

\begin{lemma}
\label{thm:u-dp}
    With $f$ any mechanism and $u = \log(f)$, $f$
satisfies $\Delta(u)$-DP, and no stronger privacy guarantee is possible.
\end{lemma}
\begin{proof}
For any neighbors $X,X'\in D$, $t\in \mathbb{R}^n$, 
$|\log(f_X/f_X')(t)| = |u(X,t)-u(X',t)|\leq\Delta(u).$ This shows $f$ satisfies $\Delta(u)$-DP. 
For any $\eta>0$, there is some neighboring $X,X'\in D$ and $t\in \R^n$ such that $|\log(f_X/f_X')(t)| = |u(X,t)-u(X',t)|\geq \Delta(u) - \eta$. This shows no stronger privacy guarantee is possible. 
\end{proof}

\begin{theorem}[Privacy Squeeze Theorem]
Let $a,b,c$ be mechanisms, and suppose $c$ is squeezed between $a,b$, i.e., $\min(a(X,t),b(X,t)) \leq c(X,t) \leq \max(a(X,t), b(X,t))$. 
Suppose $a, b$ satisfy $\varepsilon_a$--, $\varepsilon_b$--DP, respectively. 
Then $c$ is $\varepsilon_c$--DP with $\varepsilon_c  = \max(\varepsilon_a, \varepsilon_b) + \sup_{D\times \R^n} |\log(a/b)|$.
\end{theorem}
\begin{proof}
    Write $u = \log(a),v = \log(b), w = \log(c).$ 
    Since log is increasing, we have $\min\{u,v\} \leq w \leq \max\{u,v\}$.
    For neighboring $X,X'\in D$, we obtain (colors to help reader follow terms)
    \begin{equation*}
    \begin{array}{llll}
        |w(X,t) - w(X',t)| &\leq \max \{ & |u(X,t) - u(X',t)|,   \textcolor{teal}{|v(X,t) - v(X',t)|}\\
        &&\textcolor{purple}{|u(X,t) - v(X',t)}|, 
             \textcolor{violet}{|v(X,t) - u(X',t)|}  
                \} \\
            & = \max\{ & |u(X,t) - u(X',t)|,    \textcolor{teal}{ |v(X,t) - v(X',t)|},  \\  
             && \textcolor{purple}{|u(X,t) - u(X',t)  + u(X',t) - v(X',t)|}, \\
            & & 
                 \textcolor{violet}{|v(X,t)-v(X',t) + v(X',t) - u(X',t)|}\}\\
            & \leq \max\{&  \Delta(u), \Delta(v) \} + |u(X',t) - v(X',t)|
    \end{array}
    \end{equation*}
    where $\Delta(u), \Delta(v)$ are the sensitivities. 
Similarly, by using 
\textcolor{purple}{$ -v(X,t) + v(X,t) $} 
and 
\textcolor{violet}{$-u(X,t) + u(X,t)$} 
in the middle equality we obtain 
$|w(X,t) - w(X',t)|\leq  \max\{  \Delta(u), \Delta(v) \} + |u(X,t) - v(X,t)|$.  
We have $|w(X,t) - w(X',t)|\leq  \max\{  \Delta(u), \Delta(v) \} + \min_{A\in \{X,X'\}} |u_A(t) - v_A(t)|$. 
Finally, taking the supremum over neighboring $X,X'$ and all $t$ gives 
$$\Delta(w) \leq  \max( \Delta(u), \Delta(v) ) + \sup_{D\times \R^n} |u - v|\ .$$    
    It follows from Theorem \ref{thm:u-dp} that since $c = e^w$, $c$ satisfies $\varepsilon_c$--DP with $\varepsilon_c = \Delta(w) \leq \max(\varepsilon_a, \varepsilon_b) + \sup_{D\times \R^n} |u - v|.$
\end{proof}

\begin{theorem}[Projection Theorem]
\label{thm:projection}
    Suppose $a, b$ satisfy $\varepsilon_a, \varepsilon_b$-DP, respectively and $c = \alpha a + (1-\alpha)b$, with $\alpha \in [0,1]$. Then $c$ satisfies ($\max\{\varepsilon_a, {\varepsilon_b}\}$)-DP. 
\end{theorem}
\begin{proof}
    For any neighboring $X,X'\in D$ we have 
    \begin{align*}
    c_X &= \alpha a_X + (1-\alpha)b_X \\
    & = \alpha a_X' \frac{a_X}{a_X'} + (1-\alpha)b_X' \frac{b_X}{b_X'} \\
    & \leq \alpha a_X' e^{\varepsilon_a} + (1-\alpha) b_X'e^{\varepsilon_b}\\
    & \leq \max\{e^{\varepsilon_a}, e^{\varepsilon_b}\}c_X'
    \end{align*}
\end{proof}

\subsection{Mapping Squeeze Results to ExpM+NF}
\label{sec:mapping-squeeze-results}
Shown in Figure \ref{fig:expm_vary_eps}, as the $\varepsilon$ parameter increases (privacy decreases), the ExpM density $p$ moves from nearly a uniform density to a density with mass piled at the mode(s). 
In particular, the density corresponding to intermediate values of $\varepsilon$ are squeezed between higher and lower privacy ExpM densities, except for on a small set. 

This provides a new avenue for augmenting the ExpM+NF process to ensure a privacy bound on $q$, the NF output density, namely constraining the NF to lie between two ExpM densities, with different $\varepsilon$ values.
(Here we use the subscripts $_t$ and $_z$ to denote the random variable of the density, letting $p_t^{\varepsilon_i}$ denote the ExpM target density with privacy $\varepsilon_i$,  $t = g(z, \beta_{NF})$ be our NF model, $q_t$  the output density of the NF, and $r_z$ the base density for the NF; i.e., $q_t(t) = r_z(g(z))|\det J_g(z)|$ where $g(z) = t$.) 
Fix $\varepsilon_2 > \varepsilon_1 > 0$. 
We propose to force $q_t$, to satisfy the constraint $\min (p^{\varepsilon_1}_{t} , p^{\varepsilon_2}_{t}) \leq q_t \leq \max (p^{\varepsilon_1}_{t} , p^{\varepsilon_2}_{t})$.
How one can achieve this constrained optimization both theoretically (i.e., with provable guarantees) and in practice is future research. 
There are two apparent obstructions:
First,  enforcing such a constraint as written involves computation of the unknown denominators of ExpM, $Z = \int e^{\varepsilon u(t,X) /(2s)}dt$.
Second, using simple regularization (or an iterative Lagrangian Relaxation optimization scheme) is not guaranteed to provide a feasible solution (a $q_X$ living within the constraint region) without further hypotheses. 
A potential route to a privacy proof is to enforce analogous constraints on the log densities via Lagrangian optimization (alternating optimization of NF parameters, then regularization parameters). This would require analysis of the unknown denominators for ExpM densities to prove it suffices. 

Similarly, if the NF output $q_X$ can be shown to be uniformly close to any element on the line segment  $[p_t^{\varepsilon_1},  p_t^{\varepsilon_2}]$, then using the Projection Theorem \ref{thm:projection} with Theorem \ref{thm:convergence-dp} or \ref{thm:uniform-convergence} furnish the needed guarantee. 

%% file: 80-conclusion.tex
\section{Conclusion \& Future Directions}
\label{sec:conclusion}
The fundamental observation driving this new direction is that differentially private ML is constrained by an accuracy-privacy tradeoff that is often unusable in practice. 
Examining DPSGD, we find that the algorithm is wasteful by design in all but the singular scenario when gradients are to be made public. 
Our critical insight that the well-established but often intractable ExpM is a tailored optimization solution, and we now have new sampling techniques potentially allowing its use. 
By design ExpM circumvents the step--wise decline in privacy inherent to DPSGD, so research for unlocking ExpM provides a key to advancing the field. 
To this end, we research use of NFs to sample from the intractable ExpM density. 
Perhaps unexpectedly, we exhibit that a single sample from an NF trained to approximate the ExpM density is nearly as accurate as non-private SGD, and more so than DPSGD. 
This advancement shows that pairing ExpM with a modern sampling method is a viable solution in terms of computational expense and accuracy. 
Such a method could address many problems in practice, e.g., a decreased privacy cost will afford needed hyperparameter training runs that ultimately will provide even larger steps in accuracy. 
Promises aside, exhibiting a privacy proof for ExpM+NF has thus far eluded us. 
Implementation of state--of--the--art privacy auditing methods provide some evidence that ExpM+NF adds privacy over non-private SGD, but not as much as DPSGD. 
En route to these results we exhibit ancillary benefits, including state--of--the--art privacy and accuracy results on MIMIC-III benchmark, bayesian learning feasibility of ExpM+NF, and many Theorems on proving privacy in the context of density learning.

At first glance, the accuracy results of ExpM+NF (Section \ref{sec:accuracy-results}) may seem a bit too good to be true. 
First recall that, as we have not proven privacy, $\varepsilon$ is simply a parameter that governs the variance of the target ExpM density. 
Second, refer back to the Illustrative Example section \ref{sec:example}. 
Our evidence is that the ExpM+NF output density is very precise---it puts of most its mass near the mode of the true ExpM density. 
This allows for a single sample to give very--close--to--optimal parameters, as seen in our results,  even for very small $\varepsilon$)!
How much privacy is provided by ExpM+NF is unknown. 
Indeed, lower variance in the density, as we are led to believe is true of the NF output density, means worse privacy for the ExpM density. 
On the other hand, as discussed in Section \ref{sec:converse} such an NF density may indeed exhibit strong privacy, as its privacy is a function of the NF sensitivity to any one  training sample. 
Finally, our empirical privacy results, at least for LR, imply that ExpM+NF is more private than non-private SGD, but not so much as DPSGD. 
The question of what is the actual privacy for the parameter $\varepsilon$ persists. 
Answering this paired with the results in Section \ref{sec:accuracy-results} will allow a direct comparison of DPSGD's and ExpM+NF's privacy--accuracy curves.

As discussed in the Limitations Section \ref{sec:limitations}, the primary need for a privacy proof of ExpM+NF is a theorem that guarantees, independent of the training dataset, proximity of the NF output density to a mechanism with known privacy. 
(Specific avenues for such a proof are the topic of Section \ref{sec:math}). 
Of course, there are many avenues for tweaking the proposed ExpM+NF method to pursue this result. 
For instance, the type of NF used may be adapted to admit easier analysis. 
Note that this is wider than simply defining each flow layer as we did; novel approaches to implementation, such as Neural ODEs \cite{chen2018neural} bring to bear whole mathematical tool suites for alternative avenues of proof. 
Similarly, many new---and perhaps just forgotten---sampling techniques (e.g., \cite{neal2001annealed, restrepo2021homotopy}) that admit an error bound  may pose a solution in lieu of NFs.

Our hope is that the steps forward presented here, as well as documentation of our obstructions are fodder for the research community to successfully advance the state of the art. 

%% file: 90-appendix.tex
\sloppy 

\section{Further Background on Differential Privacy}
\label{sec:appendix-background}
Intuitively, differential privacy measures how much an individual data point changes the output distribution of the mechanism $f$. 
If the mechanism's output is nearly identical in both cases (when any data point is or isn't included), meaning that $\varepsilon$ is very small, then this is a privacy guarantee for any individual data point. 
This quantifies the intuitive notion that it is essentially indistinguishable if any one data point is included or not. 

Differential privacy has many useful properties. 
First, post-processing is ``free'', meaning if $f(\theta, X)$ is $(\varepsilon, \delta)$-DP, the same guarantee holds after composition with any function $h$; i.e.,  $h\circ f$ is also $(\varepsilon, \delta)$-DP \cite{dwork2014algorithmic}. 
Secondly, it satisfies composition theorems; the most basic of which says that if $f_1$ and $f_2$ satisfy $(\varepsilon_1, \delta_1)$-DP and $(\varepsilon_2, \delta_2)$-DP, respectively, then $f_1 \circ f_2$ satisfies $(\varepsilon_1 + \varepsilon_2, \delta_1+\delta_2)$-DP \cite{dwork2014algorithmic}. 
The general problem of estimating the overall privacy bound (a single $\varepsilon, \delta)$) for a sequence of mechanisms $f_i(\theta_i, X), i = 1, ..., m$ each furnished with a privacy guarantee $(\varepsilon_i, \delta_i)$-DP, is important for use. 

As the basic composition theorem above holds, it is not a tight bound (better estimates may exist), and hence, a large body of work \cite{dwork2007sequential, dwork2010advanced, kairouz2015composition, abadi2016deep, dong2021gaussian, gopi2021numerical} provide advanced composition theorems (tradeoffs of $\varepsilon$ and $\delta$ that depend on the $\varepsilon_i, \delta_i$) that estimate this overall bound.
Intuitively, as more information is released (as $m$ grows), the overall privacy guarantee gets worse ($\varepsilon$ and $\delta$ increase), and all advanced composition theorems entail an increase of $\delta$ for the better estimate of $\varepsilon$. 

\section{Further Background on Normalizing Flows}
\label{sec:appendix-nf}
Let $z\in \mathbb{R}^d,\  A\in \mathbb{R}^{d\times m},\ B \in \mathbb{R}^{m\times d},\ c \in \mathbb{R}^m$, and $h:\mathbb{R}\to\mathbb{R}$ an activation function that we ambiguously apply to vectors componentwise ($h(\vec{v}) = [h(v_1), ..., h(v_d)]^t$).   
\begin{definition}[Sylvester Flow]
\label{defn:sylvester}
A Sylvester flow is defined as $g(z) = z + A h(Bz + c)$, and note that 
$J_g(z) = I_d + A_{d \times m} \text{diag}(h'(Bz+c))_{m \times m} B_{m \times d}.$ 
\end{definition}
\begin{theorem}[Sylvester a.k.a. Weinstein–Aronszajn Theorem]
\label{thm:wa}
$ \det (I_d + A_{d \times m} B_{m \times d}) = \det (I_m + B_{m \times d} A_{d \times m})$. 
\end{theorem}
It follows that when $m < d$, we can compute the determinant, $J_g(z)$ 
in the lower dimensional representation!
\begin{definition}[Planar Flow]
\label{defn:planar}
A planar flow is a special case of a Sylvester flow when $m=1$, so $A, B\in \mathbb{R}^d$. 
\end{definition}

Following Berg et al. \cite{berg2018sylvester}, a Sylvester Flow can be implemented by using the special case where 
$A = Q_{d\times d} R_{d\times m}$ and $B = \hat{R}_{m\times d}Q^t_{d\times d}$ with $Q$ orthogonal and $R, \hat{R}$ upper triangular. 
(This is a restriction in that it uses the same $Q$ for both $A$, $B$). 
To instantiate an orthogonal matrix $Q$, we define unit vectors $v$ and set $Q:= \prod_v I-2vv^t$ which is a product of Householder reflections. A well know fact is that Householder reflections are orthogonal, and any orthogonal matrix is as a product of Householder reflections \cite{sun1995basis}.
Note this is obviated in the Planar Flow case since $A, B$ are simply $d$--dimensional vectors. 

\section{Further Details of ExpM+NF}\label{sec:appendix-expm}
Our utility must be $u=\phi \circ L$, a decreasing function of the loss. 
While any strictly decreasing $\phi$ theoretically ensures identical optimization, in practice $\phi$ must separate the near-optimal $\theta$ from those that are far from optimal. 
This means the derivative must be a sufficiently large negative value over an appropriate range of $L(\theta)$ to discriminate high versus low utility $\theta$. 
In particular, using a negative sigmoid function, which is flat ($\phi' \approxeq 0$) except in a limited neighborhood, is unwise, unless we know which range of $L(\theta)$ will contain the optimal values a priori and shift $\phi$ to discriminate these values. 
Lastly, we note that choice of $\phi$ effects the sensitivity, and therefore may be designed to assist in producing the required sensitivity bound.

When $s$ is large, it diminishes influence of $u$ on $p_{ExpM}$ and increasing the variance. This leads to higher likelihood of sampling less optimal $\theta$. 
Since $u, s$ are intimately tied to the loss function $L,$ bounding $s$ is generally required per loss function and can depend on the model outputs.
Hence, if for a particular model and loss function a strong bound can be made, this can give ExpM+NF a boost, and conversely, if only a weak or no sensitivity bound can be identified, ExpM+NF will be hindered or useless. 



\section{MIMIC-III ICU Benchmarks Experiment Details}
\label{sec:appendix-experiment-details}


For DPSGD we use PRV accounting we use clipping parameter, \texttt{max_grad_norm} = 1 \cite{pytorchOpacus, yousefpour2021opacus}, $\delta = 1\mathrm{e}{-5}$, roughly the inverse of the number of training data points, as are standard, and \texttt{delta_error} set to $\delta/1000$ (default).
We compute the \texttt{noise_multiplier} parameter, which governs the variance of the Gaussian noise added to each gradient, to achieve the target privacy $(\varepsilon, \delta)$ based on the number of gradient steps (\texttt{epochs} and \texttt{batch_size} parameters) and the gradient clipping parameter \texttt{max_grad_norm}. 
The PRV Accounting method introduces an error term for $\varepsilon$ that 
bounds the error in the estimate of the privacy guarantee. 
Notably, \texttt{eps_error =.01} is the default fidelity, and computations become too expensive or intractable for \texttt{eps_error <.0001}.
Hence, given target epsilon $\bar{\varepsilon}$, we set \texttt{eps_error} to be as close as possible to $\bar{\varepsilon} / 100$ within the range $[.0001, .01]$. 

Data is split into a train/development (dev)/test split that is stratified, preserving class bias in each partition. 
For each $\varepsilon$, we performed a randomized grid search of hyperparemeters fitting on the train set and testing on the dev set. 
The hyperparmeters were refined usually twice for a total of about 60-90 hyperparmeter train/dev runs; this same procedure is used for all training methods (ExpM+NF, DPSGD, non-private).
Only those hyperparameters producing the best results on the dev set are used in final results--fitting on training plus dev partitions and validating on the never-before-used test set. 
While in practice, hyperparameter searching will entail a loss of privacy, we ignore accounting for this as our goal is to  compare the best achievable accuracy and privacy of ExpM+NF, DPSGD, relative to the nonprivate baseline. 

For DPSGD, 
the usual hyperparameters (epochs, batch size, momentum, learning rate, and regularization parameter) function as usual. 
For the NF models, the number of flows and parameters of the Sylvester model (Sec. \ref{sec:nf}, \ref{sec:appendix-nf}) are also  hyperparameters. 
Lastly, our method has hyperparmeters for training that include: epochs, NF batch size (how many model parameters to sample from the NF for each KL loss computation), data batch size (number of training data points used in each KL loss computation), learning rate, momentum, and regularization parameter. 

While batch normalization is commonplace in non-private deep learning, it is not directly portable to the private setting as $z$-normalization learned in training must be shipped with the model \cite{yousefpour2021opacus}; 
hence, we do not use batch normalization for our private models, but we do provide results with vs. without batch normalization where applicable, namely for the non-private GRU-D model, Table \ref{tab:nonprivate-bn}. 
Details of the hyperparmeters used for each model can be found in our code base and vary per experiment and privacy level $\varepsilon$.

To address anomalous starting conditions, for both DPSGD and non-private training with BCE and $\ell^2$ loss, once hyperparameter searching is complete, we train ten non-private baseline models, each with a different random seed, and present the median AUC. 
The highest non-private median across the two loss functions is the considered the baseline for each dataset, task and appears as a dotted line in result plots. 
For ExpM+NF  we train ten NFs using the best hyperparameters and produce 1,000 sampled parameter sets from each NF, for a total of 10,000 parameter sets.  
We instantiate the 10,000 resulting models, compute their AUC results, and compare the median to DPSGD and the non-private baseline.
One may prefer to compute the median of the 1,000 AUCs per NF model (producing 10 median AUCs), and report the median of these medians. 
We have empirically verified that the results are nearly identical in our experiments. 

\subsection{MIMIC-III Dataset Details}
\label{sec:appendix-mimic}
MIMIC-III \cite{johnson2016mimic} includes records for a diverse group of critical care patients at a medical center from 2001–2012. 
Our experimental design follows a progression of research on MIMIC-III. 
Wang et al. \cite{wang2020mimic} create benchmarks of needed healthcare predictive tasks from the MIMIC-III dataset, downloadable (from MIMIC-III v1.4, as \texttt{all_hourly_data.h5}) with the codebase \cite{mimic-extract}. 
We use the binary prediction task benchmarks ICU Mortality and ICU Length of Stay ($>$ 3 days) from Wang et al.'s pipeline. 
These tasks differ only in target and share the feature set consisting of 104 measurement means per hour over a 24 hour period for 21,877 patients. 
Mortality is very imbalanced (07.4\%/92.6\%), while Length of Stay is balanced (47.1\%/52.9\%). 
We employ standardization ($z-$score normalization) for each feature, and it is applied independently per train/dev/test partition to prevent privacy issues.

As an example use of the ICU Mortality and Length of Stay benchmarks, Wang et al. also present results of non-private LR and GRU-D \cite{che2018recurrent} models.
Che et al. \cite{che2018recurrent} introduce the Gated Recurrent Unit with Decay (GRU-D) model, a deep neural network for time series learning tasks.  
GRU-D interpolates between the last seen value and the global mean so censored (missing) data does not need to be imputed in advance. 
Additionally, the hidden units of the GRU are decayed to zero with learned decay per component, which allows features that should have short influence (e.g., heart rate) vs. long influence (e.g., dosage/treatment) to have learned rate of diminishing value. 
Wang et al. provide non-private results for LR and GRU-D models on these two tasks as well as implementations. 

Suriyakumar et al. \cite{suriyakumar2021chasing} replicate Wang et al.'s results on these two tasks for non-private baselines, then provide DPSGD (with RDP accountant) results as well.  
To the best of our knowledge, Suriyakumar et al. \cite{suriyakumar2021chasing} present the state of the art in differentially private ML on these MIMIC-III benchmarks.
Unfortunately, their code is not provided. 

Using the currently available data (MIMIC-III v1.4, as \texttt{all_hourly_data.h5}), and following preprocessing code and descriptions of Wang et al. we replicate these benchmarks as closely as possible. 
PRV accounting \cite{gopi2021numerical} provides a better privacy bound than RDP, and we present comparative results across the two accounting methods. 
There are discrepancies between our dataset and Suriyakumar et al.'s description of the data, possibly because MIMIC-III was updated. 
The current MIMIC-III ICU Mortality and Length of Stay benchmarks have 89.1\% missingness rate and 104 features per hour $\times$ 24 hours giving 2,496 columns. 
Suriyakumar et al. claim to have 78\% missingness and only 78 features per hour, but they report the number of columns is reported as ``24,69'', seemingly a typo (if the correct is 2,496, would entail 104 feature per hour). 
The number of rows matches exactly. 
The LR models entails privately learning these 2,497 parameters. 
We use RMSProp optimizer with a scheduler that reduces the learning rate for all LR models.

For GRU-D, we follow Suriyakurmar et al. \cite{suriyakumar2021chasing} and when needed the original GRU-D work of Che et al. \cite{che2018recurrent}. 
Suriyakumar et al. \cite{suriyakumar2021chasing} used batch normalization (BN) even when training privately with DPSGD, but this has since been revealed as a source of privacy leakage  \cite{yousefpour2021opacus}. 
Hence, we do not use BN with our DPSDG implementations, but  we do train a non-private model both with and without BN applied to the top regressor layer.
We apply a dropout layer to the top regressor layer, treating the dropout probability as a hyperparameter. 
Unlike \cite{suriyakumar2021chasing}, we treat the hidden size as a hyperparameter for all experiments. 
The main motivation for this choice was that memory constraints on the GPU hardware used to perform ExpM+NF required that we used a hidden size no larger than 10.  For the non-private and DPSGD implementations, we allow the hidden size to vary between 10 and 100 for our randomized grid search, so this constraint was only limiting for our new method. 
For training non-privately and with DPSGD, we use the Adam optimization method and for the NF model we use RMSProp with a scheduler.
This is our the best possible replication of Wang et al. and Suriyakurmar et al.'s experiments save the choice to omit batch normalization. 

Notably, Che et al \cite{che2018recurrent} also provide non-private GRU-D results on MIMIC-III data for an ICU Mortality prediction task. Since Che et al. do not follow the preprocessing of Wang et al. and specifically claim to use 48 hours of data (while we, Wang, Suriyakurmar use 24 hours), we do not consider their results comparable.

We provide benchmark timing results for all three training methods on MIMIC experiments in section \ref{sec:times}.

\section{Further MIMIC-III Experiment Results}
\label{sec:appendix-more-results}
\subsection{NF Distributions \& Discussion}
Recall our method is to train an NF model to approximate the ExpM density, which is a density over $\theta$, the parameters of our desired classifier. 
Our results use privacy values for ExpM+NF assuming the NF output is the ExpM distribution. 
While we cannot visualize the NF's output distribution (because it is high dimensional), we can  investigate the distribution of AUCs attained by ExpM+NF models for each  $\varepsilon$ parameters. 
We present box plots which indicate the distribution's median (middle line), first (Q1) and third (Q3) quartiles (box), and extend 1.5 times the inner quartile range (Q1 - 1.5 IQR, Q3+ 1.5 IQR, whiskers). 
To see anomalous models arising from fortunate/unfortunate starting positions (different random seeds), we overlay each of the ten model's means. 
We also include visualizations of the distribution of the AUC values produced by individual NFs. 
See Fig. \ref{fig:distribtions_grud_icu}. Results are similar across LR and GRU-D, and across both MIMIC-III prediction tasks. 
Randomness due to the model initialization is exhibited in the variance of each $\varepsilon$'s bar chart. 
Noise in the hyperparameter search process may influences the distribution trend across $\varepsilon$ (as a different hyperparmeter search is performed for each $\varepsilon$). 
For this method specifically, outlier NFs can produce bad results based on their randomized initializations.  
The increase in accuracy and privacy by ExpM+NF also affords a straightforward workaround, e.g., releasing ten models, each sampled from trained NFs with different starting seeds that are then used in a voting scheme.

\begin{figure*}
    \centering
    \includegraphics[width=.53\linewidth]{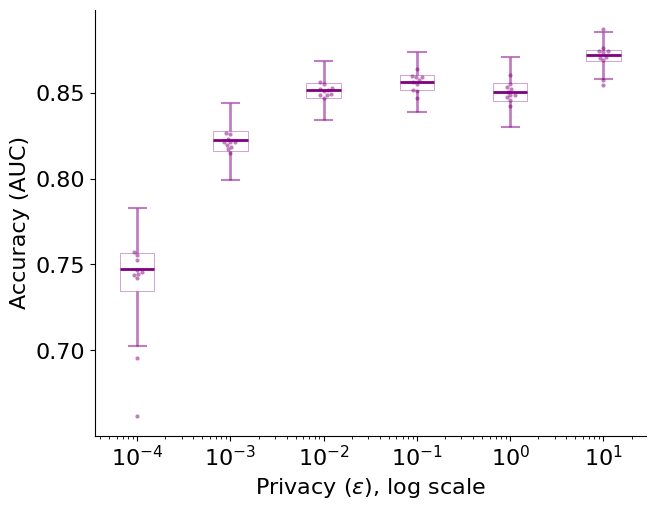}    \includegraphics[width=.45\linewidth]{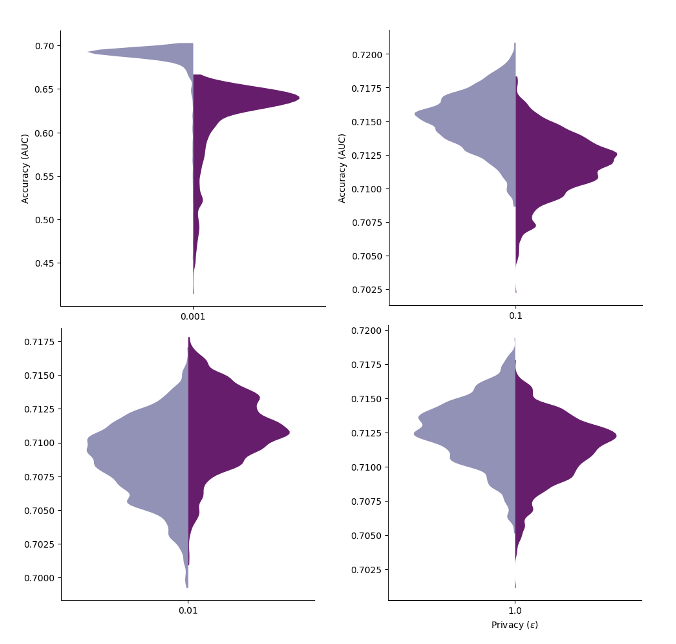}
    \caption{For GRU-D model on ICU Mortality task, for each $\varepsilon$ value, ten NF models are trained with identical hyperparmeters (but different starting seeds) and used to produce 1,000 samples each from the ExpM distribution. 
    \textbf{Left} presents box plots of the 10,000 resulting classifiers' AUC values. For each  NF model's 1,000 samples we compute the mean AUC, which are overlayed as dots. 
    \textbf{Right} presents the AUC distribution from two NFs for four $\varepsilon$ values.
    } 
    \label{fig:distribtions_grud_icu}
\end{figure*}

We found that $\sigma$, the hyperparameter governing the variance of the NF's base distribution, has a large influence on results. 
Optimal $\sigma$ found in hyperparameter searches tend to be larger for smaller $\varepsilon$, matching the intuition that the NF is more accurate if the base and output distributions' variance vary together.

The effect of using Sylvester flows (over planar flows) is increased expressiveness, but it comes at the cost of \textit{many} more parameters in the NF model. 
For example, on the MIMIC-III ICU tasks we have dimension about $d= 2,500$. 
With planar flows, the number of NF parameters needed to train an LR model is  $(2d + 1) \times \text{number of flows}$, which is about 50,000 for 10 flows. 
With Sylvester flows the number of parameters is roughly $(d\times k + m^2 + m + d) \times \text{the number of flows}$, where $k \leq d-1$  is how many Householder reflections to use and $m \leq d$  (for planar flows, $m = 1$). If we use $m = k = 5$ and  10 flows we have about $150,000$ NF parameters. 
Moving to the GRU-D models, we increase $d$ to the $\mathcal{O}(50K)$, so the number of parameters explodes. 
Future research to understand the gain to accuracy-privacy vs the computational expense incurred by using more expressive NFs for ExpM+NF is needed.

\subsection{Non-Private \& DPSGD Advancements }
\label{sec:appendix-nonprivate-advancementa}
\label{sec:appendix-dpsgd-advancements}

Tables \ref{tab:nonprivate-lr} \& \ref{tab:nonprivate-grud} present non-private LR \& GRU-D results, respectively. 
We find show no significant difference between BCE and $\ell^2$ loss for  non-private and DPSGD-trained models.
(DPSGD BCE vs. $\ell^2$ results in Table \ref{tab:mimic_los_mort_results}).
Note that Wang et al.'s LR used SciKit Learn \cite{scikit-learn} while we use PyTorch. 

\begin{table*}[h]
\centering
\caption{AUC scores for non-private LR model results for this and prior works on MIMIC-III tasks. 
\textbf{Bold} results are best per task. \emph{Underlined} are the best result in this work. }
\label{tab:nonprivate-lr}
\begin{tabular}{l c c cc}
\toprule
 &  \cite{wang2020mimic} & \cite{suriyakumar2021chasing} & Our $\ell^2$  loss &  Our BCE loss \\
\cmidrule(lr){2-2} \cmidrule(lr){3-3}  \cmidrule(lr){4-4}\cmidrule(lr){5-5}
LR, ICU Mortality  &  \textbf{.887} & .82 & \emph{.839} & .83 \\
LR, Length of Stay &  \textbf{.716} & .69 & .675 & \emph{.681} \\
\bottomrule
\end{tabular}
\vspace{-.3cm}
\end{table*}

 \begin{table*}[h]
\centering
\caption{AUC scores for non-private GRU-D model from this and prior works on MIMIC-III tasks. \textbf{Bold} results are best per task. \emph{Underlined} results are best in this work per task. BN denotes batch normalization.}
\label{tab:nonprivate-grud}\label{tab:nonprivate-bn}
\begin{tabular}{l cc cc cc}
\toprule
    & \multirow{2}{*}{\cite{wang2020mimic}}& \multirow{2}{*}{\cite{suriyakumar2021chasing}}  &  \multicolumn{2}{c}{Our $\ell^2$ loss} & \multicolumn{2}{c}{Our BCE loss} \\

&                 &                                & BN & No BN & BN & No BN \\
       \cmidrule(lr){2-2} \cmidrule(lr){3-3}  \cmidrule(lr){4-5}\cmidrule(lr){6-7}
       \cmidrule(lr){4-5}
GRU-D, ICU Mortality  & \textbf{0.891} & 0.79 & 0.855 & 0.875 & 0.877 & \emph{0.878} \\
GRU-D, Length of Stay & \textbf{0.733}  & 0.67  &  0.721 & \emph{0.724} & 0.722 & 0.722\\
\bottomrule
\end{tabular}
\vspace{-.3cm}
\end{table*}

Table \ref{tab:nonprivate-bn} compares results of the GRU-D non-private baseline for previous works and ours, specifically including BCE vs $\ell^2$ loss and with batch normalization vs. without. 
We see nearly identical results in our results for each prediction task across the the two loss functions and with/without batch normalization, except for one instance, $\ell^2$ loss with batch normalization is worse. 
GRU-D non-private baseline results of our method compared to previous works reside in Table \ref{tab:nonprivate-grud}. 
Our results are very close but slightly below Wang et al. \cite{wang2020mimic}, yet better than Suriyakumar et al. \cite{suriyakumar2021chasing}. 
The gains in GRU-D's performance of ours and Wang et al.'s results may be due to differences in data or other implementation details. Specifically, our hyperparameter search found that a hidden size of about 35-40 performed best, which is much smaller than the 67 hidden size of previous works \cite{suriyakumar2021chasing, che2018recurrent}.

Our results show GRU-D, a much more expressive model than LR, does indeed surpass LR in AUC for non-private and each private training method in all experiments, respectively. 
This trend is seen in Che et al.\cite{che2018recurrent} and Wang et al. \cite{wang2020mimic}, but results are the inverse (LR is more accurate) in results of Suriyakumar et al. \cite{suriyakumar2021chasing}.

We found substantial limitations to the DPSGD RPD accountant compared to the SOTA accounting method, PRV. 
By using the PRV accounting method, we achieve state-of-the-art results for DPSGD with both LR and GRU-D models on both ICU Mortality and Length of Stay tasks over the previous works' health care prediction results \cite{wang2020mimic, suriyakumar2021chasing}. 
Suriyakumar et al. \cite{suriyakumar2021chasing}  report LR model AUC for ``low privacy'' ($\varepsilon = 3.5 \mathrm{e}5$, $\delta = 1\mathrm{e}{-5}$) and ``high privacy'' ($\varepsilon = 3.5$, $\delta = 1\mathrm{e}{-5}$) on both ICU Mortality and Length of Stay tasks. 
Privacy of $\varepsilon>10$ is widely considered ineffective \cite{ponomareva2023how}; hence, the ``low privacy'' $\varepsilon = 3.5\mathrm{e}5$ is practically unreasonable but perhaps useful for investigating sensitivity of this parameter. 

Our replication of the same experiment produced LR models trained with DPSGD that greatly improve on this prior state-of-the-art. 
On ICU Mortality, Suriyakumar et al. \cite{suriyakumar2021chasing} report 60\% AUC for $(3.5, 1\mathrm{e}{-5})$-DP, whereas we are able to achieve 76\% AUC for the same $(3.5, 1\mathrm{e}{-5})$-DP. 
Our DPSGD implementations are able to achieve $72-76\%$ AUC on this task for $0.5 \leq \varepsilon \leq 10$ and maintain AUC above 60\% until $\varepsilon<1\mathrm{e}{-1}$. For the Length of Stay task, 
Suriyakumar et al. \cite{suriyakumar2021chasing} report 60\% AUC  for $(3.5, 1\mathrm{e}{-5})$-DP, yet our models have 67\% AUC for the same $\varepsilon$, which is only slightly below the non-private baseline. 
Again, we do not see an AUC below 60\% until $\varepsilon < 1\mathrm{e}{-1}$. Much of this improvement can be attributed to using the PRV Accounting method, which allowed us to use much less noise to achieve the same level of privacy.

For the Length of Stay task, Suriyakumar et al. \cite{suriyakumar2021chasing} report an AUC of 61\% for GRU-D with $(3.5, 1\mathrm{e}{-5})$-DP. 
We reach 70\% AUC with DPSGD for the same level of privacy. This is better than the AUC reported by Suriyakumar et al. \cite{suriyakumar2021chasing} for both the non-private LR and GRU-D models and we are able to maintain an AUC greater than 65\% for $\varepsilon > 1\mathrm{e}{-2}$ with DPSGD. As we discussed in the prior section, much of the gains in AUC for GRU-D with DPSGD are likely due using the PRV Accounting method. 
Some of the gains could also be attributed to the same causes that provided the large gains we made for the non-private GRU-D model. 
Suriyakumar et al. \cite{suriyakumar2021chasing} show a sharp decrease in AUC when training GRU-D with DPSGD on ICU Mortality task, even in the ``low privacy'' case with $\varepsilon = 3\mathrm{e}5$. 
They report 59\% AUC for this high $\varepsilon$ and 53\% AUC for the ``strong privacy'' case, $(3.5, 1\mathrm{e}{-5})$-DP, which is not much better than random guessing. 
Our DPSGD PRV implementation of GRU-D achieves median 82.7\% AUC for $(3.5, 1\mathrm{e}{-5})$-DP using GRU-D on the Mortality prediction task. 
This is comparable to the best non-private LR model that we achieved, and is better than the performance \cite{suriyakumar2021chasing} 
reported for non-private GRU-D. 
Our DPSGD results achieve AUC greater than 60\% for each $\varepsilon > 1\mathrm{e}{-2}$ in this prediction task.

\subsection{Training Time Details}
\label{sec:appendix-times}
The details of this subsection correspond with performance results in Sec. \ref{sec:times}. 
Note that either by choice or a bug, PyTorch does not expose the ten runs' times, nor statistics beside their mean \url{https://github.com/pytorch/pytorch/issues/106801}. 
In the case of the non-private BCE loss, the number of epochs and batch size were larger, likely attributing to the increased training time. 
Raw results for each training method, task, and $\varepsilon$ appear in the repository. 


\begin{table}[h]
\centering
\caption{Bayesian Inference and MAP Estimate Results using Logistic Regression on the Mortality ICU Benchmark}
\label{tab:uq-results}
\begin{tabular}{lrrrr}
\toprule
$\tilde{\varepsilon}$ &  0.001 &  0.01 &  0.1 &    1 \\
\midrule
Bayesian AUC         &   0.79 &  0.81 & 0.80 & 0.80 \\
MAP AUC &   0.79 &  0.81 & 0.78 & 0.80 \\
\bottomrule
\end{tabular}
\end{table}

\begin{table*}[h]
\centering
\caption{Length of Stay Benchmark Timing Results}
\label{tab:los_mean_times}
\begin{tabular}{cc ccc ccc}
\toprule
& & \multicolumn{3}{c}{Logistic Regression} & \multicolumn{3}{c}{GRU-D} \\
        \cmidrule(lr){3-5} \cmidrule{6-8}
\multirow{2}{*}{Training} & \multirow{2}{*}{Loss} & Training & Computing & \multirow{2}{*}{Total} & Training & Computing & \multirow{2}{*}{Total} \\
    &   &  Time & $\varepsilon$ & &  Time & $\varepsilon$ & \\
\cmidrule(lr){1-2}              \cmidrule(lr){3-5} \cmidrule{6-8}
\multirow{2}{4em}{Non-Private} & $\ell^2$ & 1.57 ms & $-$ & 1.57 ms & .81 ms & $-$ & .81 ms\\
 & BCE & 1.51 ms & $-$ & 1.51 ms & 1.46 ms & $-$ & 1.46 ms\\
 \cmidrule(lr){1-2} \cmidrule(lr){3-5} \cmidrule{6-8}
 \multirow{2}{4em}{DPSGD} & $\ell^2$ & 1.41 ms & 1.36 ms & 2.77 ms & 1.8 ms & 1.5 ms & 3.3 ms\\
 & BCE & 1.36 ms & 1.3 ms & 2.63 ms & 1.67 ms & 1.76 ms & 3.45 ms \\ 
  \cmidrule(lr){1-2} \cmidrule(lr){3-5} \cmidrule{6-8}
ExpM+NF & RKL+$\ell^2$ & 1.74 ms & $-$ & 1.74 ms & 1.81 ms & $-$ & 1.81 ms
 
\\
\bottomrule
\end{tabular}
\end{table*}




\section{Details on Bayesian Learning and UQ with ExpM+NF}
\label{sec:appendix-uq}
Setting $p(y|x,t)  \propto \exp(-\varepsilon l(x,y,\theta)/(2s))$ and prior $p(\theta)= \exp(-r(\theta))$ for regularization function $r,$ we see  the posterior $p(\theta| X,Y) = \prod_{X,Y} p(y_x,t) p(\theta) / Z = p_{ExpM}(\theta| X, Y )$.
Sampling $\theta_i\sim p_{NF}(\theta |X,Y) $ is rapid as we simply sample a batch $z_i\sim p_z = N(0, \sigma I)$ (our base density), then push them through the NF network (without tracking gradients). 
Finally, for inference, given $x$, first note that since we are in the binary classification setting  
\begin{align*}
    \int_y yp(y|x,\theta) dy = 1 p(y = 1|x, \theta) + 0 p(y = 0|x, \theta) = p(y = 1|x, \theta).
\end{align*}
Next, we need some independence assumptions, namely that $p(y|x) = p(y|x, X,Y) $ (output $y$ is independent of the training data), and $\theta| (X,Y)$ is independent of $x, y$. 
Finally, we marginalize to compute the prediction: 
\begin{align*}
\mathbb{E}_y(y|x) &= \int_y y p(y|x) dy \\
&= \int_y y p(y|x, X,Y) dy \\
&= \int_y y \int_\theta  p(y, \theta |x, X,Y) d\theta dy \\
&= \int_y y \int_\theta  p(y|x, \theta) p(\theta|X,Y) d\theta dy \\
&= \int_\theta  p(y =1 |x,\theta)  p(\theta|X,Y) d\theta \approx \tfrac{1}{N} \sum_i p(y =1 |x,\theta_i).
\end{align*}

As a simple example, we train an LR classifier using ExpM+NF on the Mortality ICU benchmark, sample $N = 10,000$ model parameters $\theta_i$, and compare the AUC results using Bayesian inference as described above against the MAP estimate (using $\theta_{i_0} = \argmax \{p(\theta_i|X,Y)\}$). Results are in Table \ref{tab:uq-results}.

\section{Further Likelihood Ratio Attack Results}
\label{sec:appendix-lira}
Figure \ref{fig:grud_mort_icu_mia_auc} and Figure \ref{fig:grud_los3_mia_auc} show the median ROC curve for the Likelihood Ratio Attack against 50 GRU-D target models. For the ICU Mortality task, we see that all three methods (non-private, DPSGD and ExpM+NF) the attack is unsuccessful. For the Length of Stay task, the non-private baseline method was unsuccessfully attacked, as was ExpM+NF for all but $\varepsilon = .0001$. Interestingly, DPSGD was successfully attacked for $\varepsilon = 2$ and $\varepsilon =  7$, even though the non-private baseline could not be attacked. This shows that even models with a differential privacy guarantee can be successfully attacked. Much like the Logistic Regression model results, we view these as inconclusive.
\begin{figure*}
    \centering
    \includegraphics[width=\textwidth]{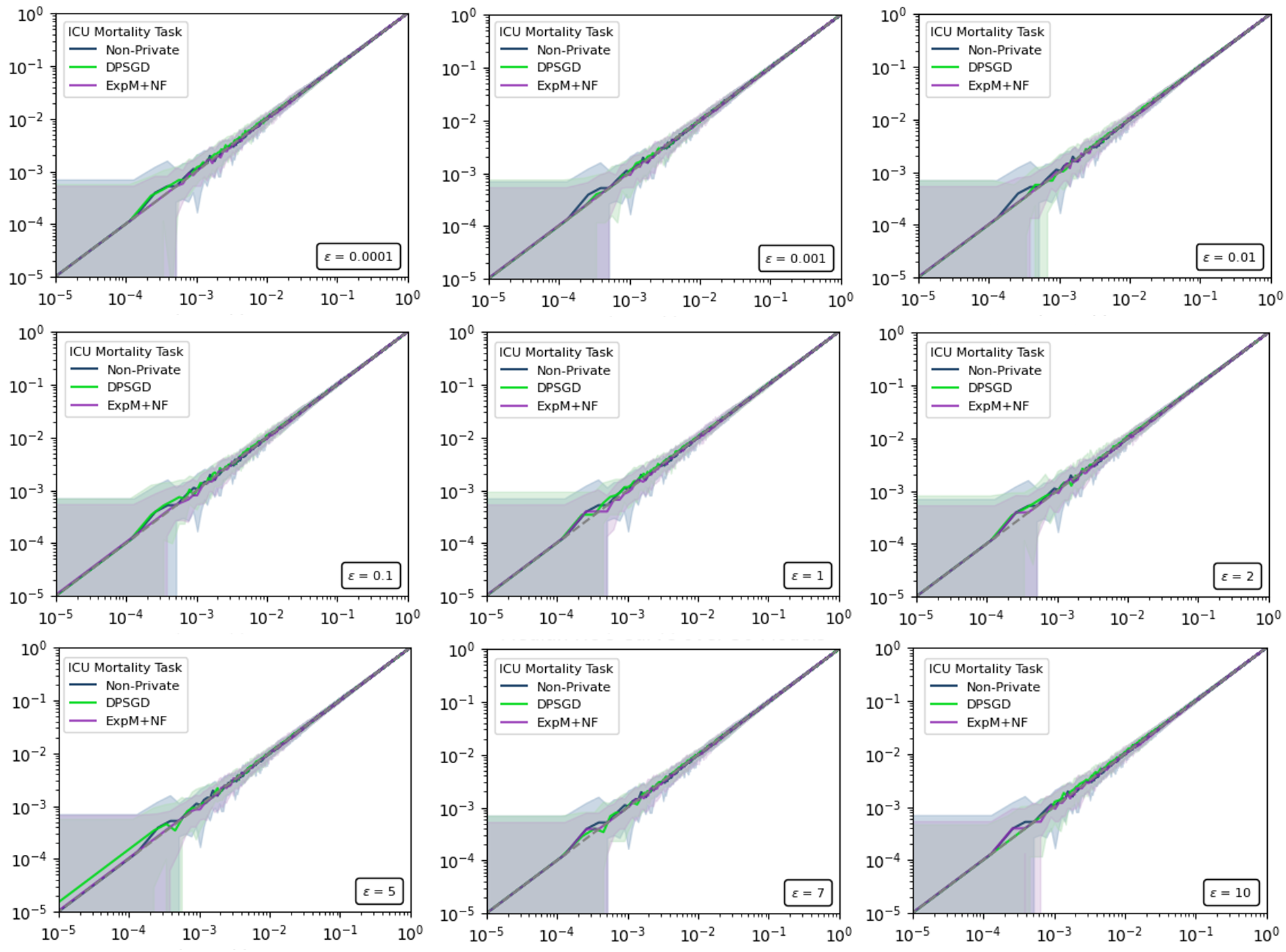}
    \caption{Likelihood Ratio Attack (median) ROC curve with 5th and 95th percentiles shaded on a log-log for GRU-D Target Models on Mortality Task.}
    \label{fig:grud_mort_icu_mia_auc}
    \vspace{-.25cm}
\end{figure*}

\begin{figure*}
    \centering
    \includegraphics[width=\textwidth]{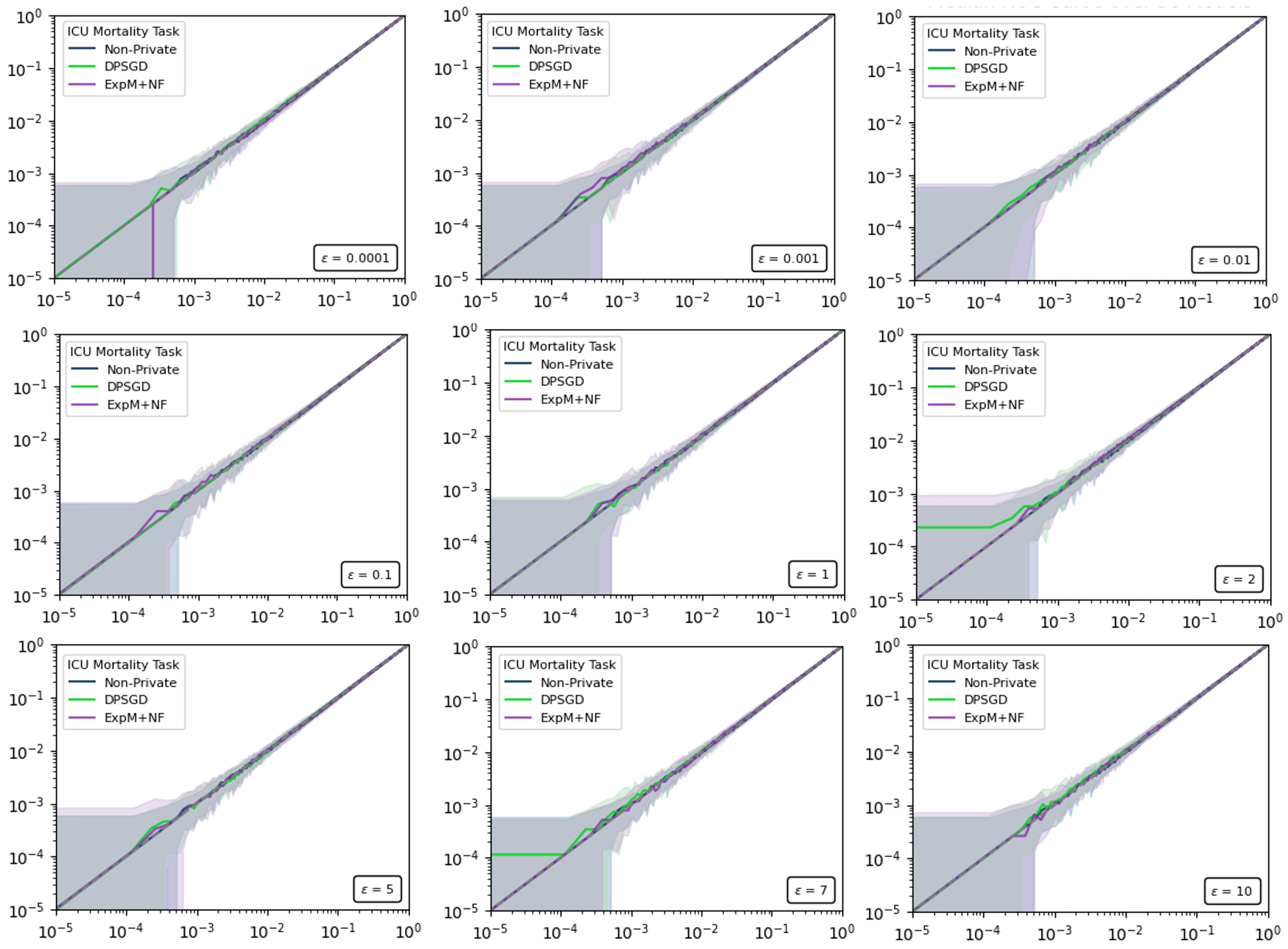}
    \caption{Likelihood Ratio Attack (median) ROC curve with 5th and 95th percentiles shaded on a log-log for GRU-D Target Models on Length of Stay Task.}
    \label{fig:grud_los3_mia_auc}
    \vspace{-.25cm}
\end{figure*}